\documentclass{article}

\usepackage[preprint,nonatbib]{neurips_2023}

\usepackage[utf8]{inputenc} 
\usepackage[T1]{fontenc}    
\usepackage{hyperref}       
\usepackage{url}            
\usepackage{booktabs}       
\usepackage{amsfonts}       
\usepackage{nicefrac}       
\usepackage{microtype}      
\usepackage{xcolor}         
\usepackage{bm}
\usepackage{amsmath}
\usepackage{amsthm}
\usepackage{bbding}
\usepackage{graphicx} 
\usepackage{float} 
\usepackage{appendix}
\usepackage{verbatim}
\usepackage{geometry}

\usepackage{subcaption}     
\usepackage{placeins}
\usepackage{wrapfig}
\usepackage{xcolor}
\usepackage{comment}

\usepackage{wrapfig}

\newtheorem{proposition}{Proposition}

\newtheorem{Definition}{Definition}

\title{
Stability Analysis Framework for Particle-based Distance GANs with Wasserstein Gradient Flow}

\author{%
	Chuqi CHEN$^{a}$\\
	\texttt{cchenck@connect.ust.hk}
	\And
	Yue WU$^{a}$\\
	\texttt{ywudb@connect.ust.hk}
	\AND
	Yang XIANG $^{a,b,}$ \thanks{Corresponding author.} \\
        \texttt{maxiang@ust.hk} \\
	$^{a}$ \small\textit {Department of Mathematics} \\
        \small\textit {Hong Kong University of Science and Technology}\\
        \small\textit {Clear Water Bay, Hong Kong SAR, China}\\
        $^{b}$\small\textit{Algorithms of Machine Learning and Autonomous Driving Research Lab}\\
        \small\textit{HKUST Shenzhen-Hong Kong Collaborative Innovation Research Institute}\\
        \small\textit{Futian, Shenzhen, China}
}
\pdfoutput=1

\begin{document}

\maketitle

\begin{abstract}
In this paper, we investigate the training process of generative networks that use a type of probability density distance named particle-based distance as the objective function, e.g. MMD GAN, Cramér GAN, EIEG GAN. However, these GANs often suffer from the problem of unstable training. In this paper, we analyze the stability of the training process of these GANs from the perspective of probability density dynamics. In our framework, we regard the discriminator $D$ in these GANs as a feature transformation mapping that maps high dimensional data into a feature space, while the generator $G$ maps random variables to samples that resemble real data in terms of feature space. This perspective enables us to perform stability analysis for the training of GANs using the Wasserstein gradient flow of the probability density function. We find that the training process of the discriminator is usually unstable due to the formulation of $\min_G \max_D E(G, D)$ in GANs. To address this issue, we add a stabilizing term in the discriminator loss function. We conduct experiments to validate our stability analysis and stabilizing method.
\end{abstract}


\section{Introduction}

Generative Adversarial Networks (GANs) \cite{refGAN} have emerged as a prominent framework for generative modeling in recent years, finding applications across a wide range of fields, including image style transformation \cite{refI2IGAN}, super-resolution \cite{refsuperreso}, and 3D object generation \cite{ref3Dgan} etc. In the GANs framework, there are two networks involved: the generator and the discriminator. The generator $G$ is trained to map a random variable, typically drawn from a normal distribution i.e., $\bm{z} \sim \mathcal{N}(\bm{0}, I)$, to samples that resemble those from the data distribution (i.e., $G(\bm{z}) \sim \mathbb{P}_{\mathrm{data}}$). The discriminator $D$, on the other hand, is trained to evaluate the scores $D(\bm{x}) \in \mathbb{R}^d$ of real or generated samples $\bm{x}$. Together, the generator and discriminator networks are trained iteratively to improve the quality of generated samples until the generator is able to produce samples that gain the same score from the discriminator. The standard formulation of GANs is given by  
\begin{equation}
    \min_{G} \max_{D} E(G, D),
\end{equation}
where $\min$ and $\max$ of objective function $E$ with respect to $G$ and $D$  are taken over the set of generator and discriminator functions. 

Within the GANs framework, different probabilistic metrics can be used to define various objective functions for different GANs. For example, the original GAN \cite{refGAN} uses the JS divergence, and the WGAN \cite{refWGAN} uses the Wasserstein distance. Other metrics include Cramér distance used by Cramér GAN \cite{refcramerGAN},  the maximum mean discrepancy (MMD) used by the MMD GAN \cite{refMMDGAN1,refMMDGAN2}, and the elastic interaction energy-based metric used by the EIEG GAN \cite{refEIEGGAN}. In this paper, we focus on the latter three GANs and introduce a unified expression for the probability density distance in these models, which we refer to as the particle-based distance. Furthermore, we analyze the stability of their training process using the Wasserstein gradient flow. 

Our motivations originate from molecular dynamics \cite{refXIANG,refLUO} where we consider samples from $\mathbb{P}_{\mathrm{data}}$ and generated samples from $\mathbb{P}_{g}$ as a system of interacting particles. The corresponding particle-based distance between the two distributions can be considered as the potential energy of this system. Under this framework, we treat the training process of the discriminator and generator as a process of evolution of the particles. We analyze the stability of the training by analyzing the density evolution equation for the particles, which is  
Wasserstein gradient flow based on particle-based distance. Our analysis shows that the training process of the discriminator is often unstable under the $\min_{G} \max_{D} E(G, D)$ formulation of GANs. To address this issue, we propose an additional stabilizing term in the discriminator loss function.

To summarize, our contributions can be stated as follows:
\begin{itemize}
\item In Section \ref{section3}, we propose a new framework for analyzing the training process of particle-based distance GANs using Wasserstein gradient flow. The training stability is determined by the corresponding perturbation evolution equation. Our analysis reveals that the training of discriminator is always unstable.
\item  To address the unstable training issue of the discriminator in these GANs, we introduce a stabilizing term in the discriminator loss in Section \ref{section4}.
\item In Section \ref{section6}, we conduct experiments to validate our analysis and the proposed stabilizing method.
\end{itemize}

Finally, we study the connection to existing works in Section \ref{section5} and discuss the potential for extending our method of proving stability using functional gradient flow to other types of GAN models.

\section{Preliminaries}

\paragraph{Notation.} In this paper, we use $p_r(\bm{x})$ to denote the probability density function corresponding to the data distribution $\mathbb{P}_{\mathrm{data}}$, $p_g(\bm{x})$ for the generated data distribution $\mathbb{P}_{g}$, and $p_f(\bm{x})$ for the distribution of the samples in feature space $\mathbb{P}_{f}$. Without ambiguity, $\nabla$ stands for $\nabla_{\bm{x}}$ for conciseness. 
For GANs, $D_{\phi}$ denotes the discriminator neural network parameterized by $\phi$, and $G_{\theta}$ denotes the generator neural network parameterized by $\theta$. Notation $\|\cdot\|$ denotes the $l_2$ norm on $\mathbb{R}^d$.

\paragraph{Generative Adversarial Networks (GANs) with particle-based distance.}  
We denote the probability density distance of Cramér GAN \cite{refcramerGAN}, MMD GAN \cite{refMMDGAN1,refMMDGAN2}, EIEG GAN \cite{refEIEGGAN} in a unified form named particle-based distance.

\begin{Definition}[Particle-based distance]
Consider two probability density functions $p(\bm{x}), q(\bm{x}):  \mathbb{R}^{n} \mapsto \mathbb{R}$, the particle-based distance between these two distributions is
\begin{equation}
E[p(\bm{x}),q(\bm{x}) ] =\int_{\mathbb{R}^{n}}  \int_{\mathbb{R}^{n}} e(\bm{x},\bm{x})(p(\bm{x}) - q(\bm{x})) (p(\bm{y}) - q(\bm{y})) d \Omega_{\bm{x}} d \Omega_{\bm{y},}
\label{Eqn: Particle_metric}
\end{equation} 
where $e(\bm{x},\bm{y})$ stands for a type of distance between $\bm{x}$ and $\bm{y}$ with $e(\bm{x},\bm{y}) \geq 0$.
\end{Definition}
The distance $e(\bm{x}, \bm{y})$ can be specified for the following GAN variants:
\begin{itemize}
    \item Cramér GAN \cite{refcramerGAN} uses a distance function given by
          \begin{equation}
          e(\bm{x},\bm{y}) = \|\bm{x}-\bm{z}_0\| + \|\bm{y}-\bm{z}_0\| -\|\bm{x}-\bm{y}\|,
          \label{Eqn: CGAN }
           \end{equation}
           for any choice of $\bm{z}_0 \in \mathbb{R}^n$. ( $\bm{z}_0 = \bm{0} $ is often chosen to simplify notation \cite{refrefcramerker}).

    \item  MMD GAN with Gaussian RBF kernel \cite{refMMDGAN1} uses a distance function given by
          \begin{equation}
           e(\bm{x},\bm{y}) = k_{\sigma}^{rbf}\left(\bm{x}, \bm{y}\right)=\exp \left(-\frac{1}{2\sigma^2}\left\|\bm{x}-\bm{y}\right\|^2\right),
           \label{Eqn: MMD Gua Kernel}
           \end{equation}
     where $\sigma$ is a scaling factor.
     \item MMD GAN with rational quadratic kernel \cite{refMMDGAN2} uses a distance function given by
           \begin{equation}
            e(\bm{x},\bm{y}) = k_{\alpha}^{rq}\left(\bm{x}, \bm{y}\right)= (1 + \frac{\|\bm{x} - \bm{y}\|^2}{2\alpha})^{-\alpha},
           \label{Eqn: MMD }
           \end{equation}
     where $\alpha$ is a scaling factor.
    \item  EIEG GAN \cite{refEIEGGAN} uses a distance function given by
           \begin{equation}
           e(\bm{x},\bm{y}) = \frac{1}{\|\bm{x}-\bm{y}\|^{n-1}},
           \label{Eqn: EIEG }
            \end{equation}
    where $n$ is the dimension of $\bm{x}$ and $\bm{y}$.
\end{itemize}
\textbf{Remark.} Given a type of distance $e(\bm{x},\bm{y}) \geq 0$ between $\bm{x}$ and $\bm{y}$, then $E[p(\bm{x}),q(\bm{x})] = 0$ if and only if $p(\bm{x}) = q(\bm{x})$.

The proposed particle-based distance can be written in the following form:
\begin{equation}
        E[p,q] = - 2\mathbb{E}_{\bm{x} \sim p}\mathbb{E}_{\bm{y} \sim q}e(\bm{x},\bm{y}) +\mathbb{E}_{\bm{x} \sim p}\mathbb{E}_{\bm{y} \sim p}e(\bm{x},\bm{y}) + \mathbb{E}_{\bm{x} \sim q}\mathbb{E}_{\bm{y} \sim q}e(\bm{x},\bm{y}).
\label{Eqn: Particle_metric_epe}
\end{equation}

In Eq. (\ref{Eqn: Particle_metric_epe}), the first term captures the interaction energy between samples from different distributions. The second and third terms, on the other hand, represent the self-energy of samples within their distributions, respectively. 

The objective function for GAN variants based on the particle-based distance is
\begin{equation}
\begin{aligned}
E(G,D) = &
-2 \mathbb{E}_{\bm{x} \sim \mathbb{P}_{\mathrm{data}}, \bm{z} \sim \mathcal{N}(\bm{0},I)}[e_{D}(\bm{x}, G(\bm{z}))]\\
&  +\mathbb{E}_{\bm{x}, \bm{x}^{\prime} \sim \mathbb{P}_{\mathrm{data}}}\left[e_{D}\left(\bm{x}, \bm{x}^{\prime}\right)\right]+\mathbb{E}_{\bm{z}, \bm{z}^{\prime} \sim \mathcal{N}(\bm{0},I)}\left[e_{D}\left(G(\bm{z}), G(\bm{z}^{\prime})\right)\right],
\end{aligned}
\label{Eqn: Objective Energy distance GAN}
\end{equation}
where $e_D(\bm{x}, \bm{y}) = e(D(\bm{x}), D(\bm{y}))$. 

\paragraph{Wasserstein gradient flow and particle dynamics.} Given a target distribution $p_{*}(\bm{x})$, and a distance between $p_{t}(\bm{x})$ and $p_{*}(\bm{x})$, i.e.,  $E[p_{t}(\bm{x}), p_{*}(\bm{x})]$, a Wasserstein gradient flow is a curve for $p_{t}(\bm{x})$ following the direction of the steepest descent of a functional
$E(p_{t}(\bm{x})) = E[p_{t}(\bm{x}),p_{*}(\bm{x})]$, which leads $p_t(\bm{x})$ converges to $p_{*}(\bm{x})$. 

\begin{Definition}[Wasserstein gradient flow \cite{refFPeqn}]
Given an energy functional $E(p_t(\bm{x}))$, the Wasserstein gradient flow of density function $p_t(\bm{x})$ is defined as
    \begin{equation}
         \frac{\partial p_t}{\partial t} = -\nabla_{W} E(p
         _t)=\nabla \cdot\left(p_t \nabla \frac{\delta E}{\delta p_t}\right),
    \label{Eqn: Wflow_pre}
    \end{equation}
   where $\nabla_{W} E(p_{t})$ is the first variation of the functional $E(p_{t})$ in  Wasserstein spaces and $\frac{\delta E}{\delta p_{t}}$ is the first variation of the functional $E(p_{t})$ in Hilbert spaces.
\end{Definition}

The Wasserstein gradient flow possesses a physical interpretation in molecular dynamics. 
When particles are initially distributed according to $X_0 \sim p_0$, the distribution $p_t(\bm{x})$ of the particles will approach $p_*(\bm{x})$ as $t \rightarrow \infty$, following the dynamics described by the equation
\begin{equation}
dX_t = - \bigg( \nabla \frac{\delta E}{\delta p_t}\bigg) dt, \quad X_0 \sim p_0.
\label{Eqn: particle flow}   
\end{equation}

Specifically, Eq. (\ref{Eqn: particle flow}) defines the evolution equation of the particle $X_t$ whose density distribution $p_t$ satisfies Eq. (\ref{Eqn: Wflow_pre}). The stability of the particle dynamics (Eq. (\ref{Eqn: particle flow})) is consistent with the stability of the evolution equation of its corresponding density distribution (Eq. (\ref{Eqn: Wflow_pre})). Based on this, we propose our framework for analyzing the training stability of particle-based distance GANs.

\section{Stability Analysis} \label{section3}
\subsection{Analysis framework}
\paragraph{Stability of training dynamics.} 
A training dynamics is stable if the perturbations appearing at some time during the training do not cause the perturbations to be magnified as the training is continued.
That is, the training dynamics is stable if the perturbations decay and eventually damp out as the training is carried forward. 
Conversely, if the perturbations grow over time, the training dynamics is unstable.
A neutrally stable training dynamics is one in which the perturbations remain constant as the training progresses.

The framework of our stability analysis is as follows. The Wasserstein gradient flow of the particle-based distance $E(p_t(\bm{x})) = E[p_t(\bm{x}),p_{*}(x)]$ is
    \begin{equation}
         \frac{\partial p_t}{\partial t} = \nabla \cdot\left(p_t \nabla \frac{\delta E}{\delta p_t}\right).
    \label{Eqn: Wflow}
    \end{equation}

Consider at a fixed point $\bm{x}_0$ with a perturbation $v$, near the point $\bm{x}_0$, $p_{t} = p_{t}(\bm{x}_0) + v$, where $v \ll 1$ is a small perturbation. 
Substituting it into the Eq. (\ref{Eqn: Wflow}), and since $v \ll 1$ we keep only the linear terms of $v$, which gives
  
 \begin{equation}
    \frac{\partial v}{\partial t} = \mathcal{A}v,
\label{Eqn perturb}
 \end{equation}   
where $\mathcal{A}$ is linear operator on perturbation function $v$. 
Through the evolution of the perturbation $v$ Eq. (\ref{Eqn perturb}), if $|v| \rightarrow \infty$ as $t \rightarrow \infty$ the dynamics is unstable. Conversely, if $|v| \rightarrow 0$, as $t \rightarrow \infty$ the dynamics is stable. If $|v|$ remains constant, the dynamics is considered to be neutrally stable.

\subsection{Training stability analysis} \label{section3.2}
In our framework, the stability analysis of the training dynamics of particle-based distance GANs is based on the evolution equation of the distribution density, which is the Wasserstein gradient flow of the particle-based distance (Eq. (\ref{Eqn: Particle_metric})).

We view the discriminator as a feature transformation mapping that projects the high-dimensional data space into a low-dimensional feature space in our framework.
To be specific, the discriminator with parameters $\phi$ maps data samples $\bm{x} \sim \mathbb{P}_{\mathrm{data}}$ and the generated samples $G_{\theta}(\bm{z}) \sim \mathbb{P}_{g}$ to samples in the feature space, represented by $D_{\phi}(\bm{x}) \sim \mathbb{P}_{f_{\mathrm{data}}}$ and $D_{\phi}(G_{\theta}(\bm{z})) \sim \mathbb{P}_{f_g}$, respectively. 
According to this understanding,
the generator $G_{\theta}$ is trained to generate samples whose distribution matches the distribution of the data in terms of feature space. 
Specifically, the generator $G_{\theta}$ maps random variables $\bm{z} \sim \mathcal{N}(\bm{0},I)$ to samples, such that $\mathbb{P}_{f_{g}}$ approximates $\mathbb{P}_{f_{\mathrm{data}}}$.
Such an understanding is also proposed in EIEG GAN \cite{refEIEGGAN}, in which the elastic discriminator maps the data into a two-dimensional feature space, while the generator is trained to minimize the elastic interaction energy-based distance between $p_{f_{\mathrm{data}}}$ and $p_{f_g}$ in  the feature space. More discussion can be found in Appendix \ref{Appendix A}

The minmax formulation of GANs $\min_{G} \max_{D} E(G, D)$ 
is usually solved iteratively using gradient descent.
We update the parameters of the discriminator network $D$ to maximize its objective function $\max_{D} \mathcal{L}_{D}$ while keeping the generator network $G$ fixed. Then, we update the parameters of the generator network $G$ to minimize the objective function $\min_{G} \mathcal{L}_G$ while keeping the discriminator network $D$ fixed. We repeat this process for several iterations until convergence or until a stopping criterion is met.

We interpret the training process of these GANs as particle dynamics. 
Starting from the objective function based on the particle-based distance (Eq. (\ref{Eqn: Objective Energy distance GAN})),
the loss function for the generator is
\begin{equation}
    \min_{G} \mathcal{L}_{G} = -2 \mathbb{E}_{\bm{x} \sim \mathbb{P}_{\mathrm{data}}, \bm{z} \sim \mathcal{N}(\bm{0},I)} e_D (\bm{x}, G(\bm{z})) +\mathbb{E}_{\bm{z},\bm{z}^{\prime} \sim \mathcal{N}(\bm{0},I)}e_D \left(G (\bm{z}), G (\bm{z}^{\prime})\right),
\label{Eqn. loss of GEN}
 \end{equation}
with a fixed $D$. 
Correspondingly, the evolution of the generated sample dynamics in feature space is
\begin{equation}
    dX_t = \left[ - 2\mathbb{E}_{\bm{y} \sim \mathbb{P}_{f_{\mathrm{data}}}} (\nabla e(X_t,\bm{y})) + 2\mathbb{E}_{\bm{y} \sim \mathbb{P}_{f_{g}}} (\nabla e(X_t,\bm{y})) \right]dt,   \quad X_0 \sim \mathbb{P}_{f_{\mathcal{N}(\bm{0},I)}}.
\label{Eqn: Gen particle flow}    
\end{equation}
The Wasserstein gradient flow is
\begin{equation}
\begin{aligned}
      \frac{\partial p_{f_{g}} }{\partial t} &= \nabla \cdot\left(p_{f_g} \nabla \frac{\delta E}{\delta p_{f_g}}\right)
      =\nabla \cdot\left(p_{f_g} \nabla \bigg(2\mathbb{E}_{\bm{y} \sim \mathbb{P}_{f_{g}}}e(\bm{x}, \bm{y})- 2\mathbb{E}_{\bm{y} \sim \mathbb{P}_{f_{\mathrm{data}}}}e(\bm{x}, \bm{y})\bigg)\right),
\end{aligned}
\label{Eqn: Gen Wg flow}
\end{equation}
where  $E = E[p_{f_g},p_{f_{\mathrm{data}}}]$ represents the particle-based distance between $p_{f_g}$ and $p_{f_{\mathrm{data}}}$, which correspond to the probability density functions of generated and real samples in feature space, respectively.
On the other hand, the loss function for the discriminator is
\begin{equation}
\begin{aligned}
   \max_{D}  \mathcal{L}_D = -2 \mathbb{E}_{\bm{x} \sim \mathbb{P}_{\mathrm{data}}, \bm{y} \sim \mathbb{P}_{g}}e_{D}(\bm{x}, \bm{y})+\mathbb{E}_{\bm{x},\bm{x}^{\prime}\sim \mathbb{P}_{\mathrm{data}}}e_{D}\left(\bm{x}, \bm{x}^{\prime}\right)+\mathbb{E}_{\bm{y},\bm{y}^{\prime} \sim \mathbb{P}_{g}}e_{D}\left(\bm{y}, \bm{y}^{\prime}\right).
\end{aligned}
\label{Eqn. Loss of Dis}
\end{equation}
With $G$ fixed, the evolution of the generated samples dynamics in feature space is
\begin{equation}
    dX_t = \left[ 2\mathbb{E}_{\bm{y} \sim \mathbb{P}_{\mathrm{data}}} (\nabla e(X_t, \bm{y})) - 2\mathbb{E}_{\bm{y} \sim \mathbb{P}_{g}} (\nabla e(X_t,\bm{y})) \right]dt,   \quad X_0 \sim \mathbb{P}_{f_{g}}.
\label{Eqn: Discri particle flow}    
\end{equation}
And the Wasserstein gradient flow is
\begin{equation}
\begin{aligned}
      \frac{\partial p_{f_{g}} }{\partial t} &= -\nabla \cdot\left(p_{f_g} \nabla \frac{\delta E}{\delta p_{f_g}}\right) 
      = -\nabla \cdot\left(p_{f_g} \nabla \bigg(2\mathbb{E}_{\bm{y} \sim \mathbb{P}_{f_{g}}}e(\bm{x},\bm{y})- 2\mathbb{E}_{\bm{y} \sim \mathbb{P}_{f_{\mathrm{data}}}}
      e(\bm{x},\bm{y})\bigg)\right).
\end{aligned}
\label{Eqn: Discri Wg flow}
\end{equation}

It is worth noticing that the only difference between the Wasserstein gradient flow for the generator (Eq. (\ref{Eqn: Gen Wg flow})) and that for the discriminator (Eq. (\ref{Eqn: Discri Wg flow})) is their evolution direction, i.e., different sign in Eq. (\ref{Eqn: Gen Wg flow}) and Eq. (\ref{Eqn: Discri Wg flow}) due to $\min_{G} \mathcal{L}_G$ and $\max_{D} \mathcal{L}_D$. This is attributed to the min-max formulation $\min_{G} \max_{D} E(G, D)$ of GANs in which if one evolution direction is stable, the other direction is unstable.

\subsection{Results} \label{section3.3}

We use the analysis framework described above to investigate the training stability of particle-based distance GANs. Central to this analysis is Eq. (\ref{Eqn perturb}), which defines the perturbation dynamics from the Wasserstein gradient flow Eq. (\ref{Eqn: Wflow}). We find that the evolution equation for the perturbation $v$ in Fourier spaces always takes the form 
\begin{equation}
        \frac{d \hat{v}}{dt} = \mp C(2\pi)^2|\bm{\xi}|^{2}\hat{v}\mathcal{F}(e(\|\bm{x}\|)).
\label{Eqn: perturbation general}
\end{equation}
The constant $C \geq 0$ is associated with $p_{f_{\mathrm{data}}}$. In the context of GANs, the negative sign indicates the perturbation dynamics of the generator, while the positive sign indicates the perturbation dynamics of the discriminator. The function $e(\bm{x},\bm{y})$ can be expressed as $e(\|\bm{x}-\bm{y}\|)$. $\mathcal{F}(e(\|\bm{x}\|))$ denotes the Fourier transform\footnote[1]{Here we define the Fourier transform of $f(\bm{x})$ as $\mathcal{F}(f(\bm{x}))(\bm{\xi})=\int_{\mathbb{R}^n}f(\bm{x}) e^{-i 2 \pi(\bm{\xi}\cdot \bm{x})} d\bm{x}$} of $e(\|\bm{x}\|)$.
We use $\hat{v}$ to denote the Fourier transform of $v$, and $\bm{\xi}$ to represent the Fourier mode.

As shown in Eq. (\ref{Eqn: perturbation general}), if $\mathcal{F}(e(\|\bm{x}\|))(\bm{\xi}) > 0$ for all $\bm{\xi}$, the training dynamics of the generator is stable, with $\hat{v} \rightarrow 0$ as $t \rightarrow \infty$, and the corresponding training of the discriminator is unstable. Conversely, if $\mathcal{F}(e(\|\bm{x}\|))(\bm{\xi}) < 0$ for all $\bm{\xi}$, the situation is reversed, with unstable generator training and stable discriminator training. If the sign of $\mathcal{F}(e(\|\bm{x}\|))(\bm{\xi})$ depends on $|\bm{\xi}|$, both the generator and discriminator training are unstable for some value of $\bm{\xi}$. This result provide valuable insights into the stability of various particle-based distance GANs, and can guide the development of new stabilizing methods.

The stability analysis results for Cramér GAN \cite{refcramerGAN}, MMD GAN \cite{refMMDGAN1,refMMDGAN2}, and EIEG GAN \cite{refEIEGGAN} based on our framework are presented in Table \ref{Table 1}. For Cramér GAN, we observe that the training of the generator is stable, while the training of the discriminator is unstable. Similarly, for MMD GAN with a Gaussian RBF kernel, the training of the generator is stable, but the training of the discriminator is unstable. In the case of MMD GAN with a rational quadratic kernel, the situation is more complex, as the function $\mathcal{F}((1 + \frac{\|\bm{x}\|^2}{2\alpha})^{-\alpha})$ takes diverse forms for different values of $\alpha$, which leads to the training stability depending on $\alpha$.  For EIEG GAN, we find that the training of the generator is stable, while the training of the discriminator is unstable. Detailed proofs and experimental results to support these analytical findings are provided in the Appendix.

To provide a clear and concise demonstration of the stability analysis under the proposed framework, we use MMD GAN with a Gaussian RBF kernel \cite{refMMDGAN1} as an example.
 
\paragraph{Example 1 (MMD GAN with Gaussian RBF kernel \cite{refMMDGAN1}).} For MMD GAN with Gaussian RBF kernel, the evolution equation of the perturbation $v$ in generator training dynamics is 
 \begin{equation}
    \frac{\partial v}{\partial t} = C \Delta  \int_{\mathbb{R}^2}(e^{-\frac{\|\bm{x}-\bm{y}\|^2}{2\sigma^2}})vd\Omega_{\bm{y}} =  C \Delta  (e^{-\frac{\|\bm{x}\|^2}{2\sigma^2}} * v)(\bm{x}),    
\label{Eqn: pertur_MMDGAN_Gau}
\end{equation}
where $\Delta$ is Laplacian operator and $C>0.$

Taking Fourier Transform on both sides of Eq. (\ref{Eqn: pertur_MMDGAN_Gau}), we have
\begin{equation}
    \frac{d \hat{v}}{dt} = -C(2\pi)^2|\bm{\xi}|^{2}\hat{v}\mathcal{F}(e^{-\frac{\|\bm{x}\|^2}{2\sigma^2}}) = - C(2\pi)^2\sigma|\bm{\xi}|^{2}\hat{v}e^{-\frac{\sigma^2\bm{|\xi|^2}}{4}}.
\end{equation}

Thus in Fourier spaces, the solution for the perturbation term $\hat{v}$ is
\begin{equation}
    \hat{v} = \hat{v}_0e^{-\big(C(2\pi)^2\sigma|\bm{\xi}|^2e^{-\frac{\sigma^2\bm{|\xi|^2}}{4}}\big)t},
\end{equation}
where $\hat{v}_0$ is the initial value of $\hat{v}$. Thus the perturbations with all $|\bm{\xi}|$ decay, i.e., $|\hat{v}| = |\hat{v}_0|e^{-\big(C(2\pi)^2 \sigma|\bm{\xi}|^2e^{-\frac{\sigma^2\bm{|\xi|^2}}{4}}\big)t} \rightarrow 0$ as $t \rightarrow \infty$, which indicates that the training dynamics for the generator is stable. 

On the other hand, in the discriminator training dynamics, following the framework we proposed, the evolution equation of the perturbation $\hat{v}$ in Fourier spaces is
\begin{equation}
    \frac{d \hat{v}}{dt} = C(2\pi)^2|\bm{\xi}|^{2}\hat{v}\mathcal{F}(e^{-\frac{\|\bm{x}\|^2}{2\sigma^2}}) =  C(2\pi)^2\sigma|\bm{\xi}|^{2}\hat{v}e^{-\frac{\sigma^2\bm{|\xi|^2}}{4}}.
\end{equation}

The solution for the perturbation term $\hat{v}$ is 
\begin{equation}
    \hat{v} = e^{\big((2\pi)^2C\sigma|\bm{\xi}|^2e^{-\frac{\sigma^2\bm{|\xi|^2}}{4}}\big)t}.
\end{equation}
Thus the perturbations with all $|\bm{\xi}|$ decay, i.e.,  $|\hat{v}| = |\hat{v}_0|e^{\big(C(2\pi)^2 \sigma|\bm{\xi}|^2e^{-\frac{\sigma^2\bm{|\xi|^2}}{4}}\big)t} \rightarrow \infty$ 
as $t \rightarrow \infty$, which indicates that the training dynamics for the discriminator is unstable.

\begin{table}[!htbp]
	\centering
	\begin{tabular}{lcccc}
		\toprule
		GAN  & e(x,y) & $\mathcal{F}(e(\|\bm{x}\|))$  &   G & D \\
		\midrule
		Carmer \cite{refcramerGAN}   &    $\|\bm{x}\| + \|\bm{y}\| -\|\bm{x}-\bm{y}\|$    & $\frac{C_n}{|\bm{\xi}|^{n+1}}$ &  \Checkmark  &  \XSolid\\ 
		\midrule
		Gaussian RBF kernel \cite{refMMDGAN1}      &  $\exp \left(-\frac{1}{2\sigma^2}\left\|\bm{x}-\bm{y}\right\|^2\right)$       &  $\sigma e^{-\frac{\bm{|\xi|^2}\sigma^2}{4}}$ &   \Checkmark   & \XSolid  \\
		\midrule
            &    $\alpha = \frac{1}{2}$    & $\frac{1}{\alpha}K_0(\xi)$  &   \Checkmark   &  \XSolid \\
        \cmidrule{2-5}
		Rational quadratic kernel \cite{refMMDGAN2} &        $\alpha = 1$& $\frac{1}{\alpha}e^{|\bm{\xi}|}$ &  \Checkmark   &  \XSolid \\
        \cmidrule{2-5}
        $(1 + \frac{\|\bm{x} - \bm{y}\|^2}{2\alpha})^{-\alpha}$              & $\alpha = 2$        & $\frac{1}{\alpha}(-|\bm{\xi}|+8)e^{|\bm{\xi}|}$  &  \XSolid   & \XSolid  \\
        \cmidrule{2-5}
               &     $\alpha = 3$    & $\dfrac{3}{4\alpha} (|\bm{\xi}|^2 - 3|\bm{\xi}|+3)e^{|\bm{\xi}|}$ &  \Checkmark   &  \XSolid \\
        \midrule 
        EIEG \cite{refEIEGGAN}                              &  $\frac{1}{\|\bm{x}-\bm{y}\|^{n-1}}$      & $\frac{1}{|\bm{\xi}|}$ &  \Checkmark  &  \XSolid \\
		\bottomrule \\
	\end{tabular}
	\caption{ Training stability analysis of particle-based GANs. The last two columns show the training stability of generator and discriminator, i.e., stable: \Checkmark unstable:  \XSolid . $n$ is the dimension of $\bm{x}$ and $C_n \geq 0$ is a constant related to dimension. $K_0(\xi)>0$ is the Bessel function. $\bm{\xi}$ is Fourier mode.}
\label{Table 1}
\end{table}

\FloatBarrier

\section{Stabilizing Method} 
\label{section4}
The above analysis of particle-based distance GANs has revealed that the sign of $\mathcal{F}(e(\|\bm{x}\|))$ is a crucial factor in determining the stability of the training process. To stabilize the unstable training process, we propose an approach that involves introducing a stabilizing term $s(\bm{x},\bm{y}) = s(\|\bm{x}-\bm{y}\|)$ into the particle-based distance. This stabilizing term can be added to the generator or discriminator loss such that $\mathcal{F}(e(\|\bm{x}\|) - \epsilon s(\|\bm{x}\|)) > 0$ for generator and $\mathcal{F}(e(\|\bm{x}\|) - \epsilon s(\|\bm{x}\|)) < 0$ for discriminator, for all Fourier mode $\bm{\xi}$.

Without loss of generality, here we focus on the case where the training of the discriminator is unstable. The same approach can be applied to the case of unstable generator training. 
To address the instability of the discriminator training, we propose adding a stabilizing term to the particle-based distance function in the discriminator loss $\mathcal{L}_D$. Specifically, we define a modified particle-based distance function $\widetilde{e}(\bm{x},\bm{y}) = e(\bm{x},\bm{y}) - \epsilon s(\bm{x},\bm{y})$, where $s(\bm{x},\bm{y})$ is the stabilizing term and $\epsilon > 0$ is a hyperparameter. 
 
The stabilized loss function $\mathcal{L}^{s}_D$ of the discriminator is 
\begin{equation}
\begin{aligned}
     \mathcal{L}^{s}_D = -2 \mathbb{E}_{\bm{x} \sim \mathbb{P}_{\mathrm{data}}, \bm{y} \sim \mathbb{P}_{g}}\widetilde{e}_{D}(\bm{x}, \bm{y})+\mathbb{E}_{\bm{x},\bm{x}^{\prime}\sim \mathbb{P}_{\mathrm{data}}}\widetilde{e}_{D}\left(\bm{x}, \bm{x}^{\prime}\right)+\mathbb{E}_{\bm{y},\bm{y}^{\prime} \sim \mathbb{P}_{g}}\widetilde{e}_{D}\left(\bm{y}, \bm{y}^{\prime}\right),
\end{aligned}
\label{Eqn Loss of sta D}
\end{equation}
where \begin{equation}
    \widetilde{e}_{D}(\bm{x},\bm{y}) = e_{D}(\bm{x},\bm{y}) - \epsilon s_D(\bm{x},\bm{y}) = e(D(\bm{x}),D(\bm{y})) - \epsilon s(D(\bm{x}),D(\bm{y})).
    \end{equation} 
Here the stabilizing term $s(\bm{x},\bm{y}) = s(\|\bm{x}-\bm{y}\|)$ can be controlled by $\epsilon >0$. 
Consequently, the evolution equation for perturbation $\hat{v}$ in Fourier space (Eq. (\ref{Eqn: perturbation general}))
becomes:  
\begin{equation}
\frac{d \hat{v}}{dt} = C(2\pi)^2|\bm{\xi}|^{2}\hat{v}\mathcal{F}(e(\|\bm{x}\|)- \epsilon s(\|\bm{x}\|)).
\label{Eqn stable Fourier evo v}
\end{equation}
By selecting an appropriate form of the stabilizing term $s(\bm{x},\bm{y})$ and parameter $\epsilon$, we can ensure that $\mathcal{F}(e(|\bm{x}|) -\epsilon s(|\bm{x}|))(\bm{\xi}) < 0$ for any $\bm{\xi}$. This condition guarantees that $|\hat{v}| \rightarrow 0$ as $t \rightarrow \infty$, indicating that the training process of the discriminator becomes stable.

\begin{figure}[!hbtp]
	\begin{subfigure}[t]{0.24\textwidth}
		\centering
		\includegraphics[width= 1.1\textwidth]{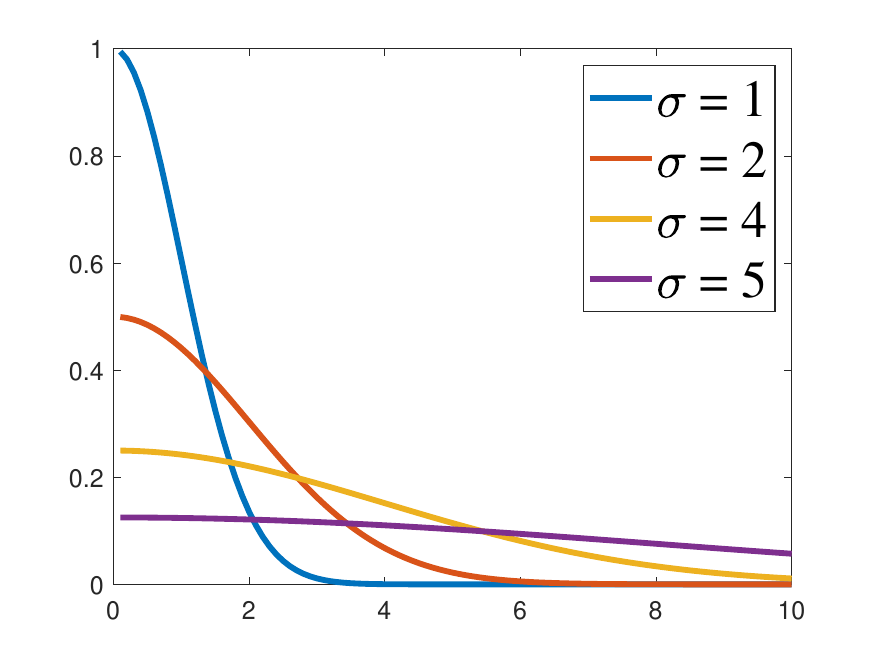}
		\caption{$\frac{1}{\sigma} e^{(-\frac{r^2}{2\sigma^2})}$.}
		\label{Fig.sub.1}
	\end{subfigure}
	\begin{subfigure}[t]{0.24\textwidth}
		\centering
		\includegraphics[width= 1.1\textwidth]{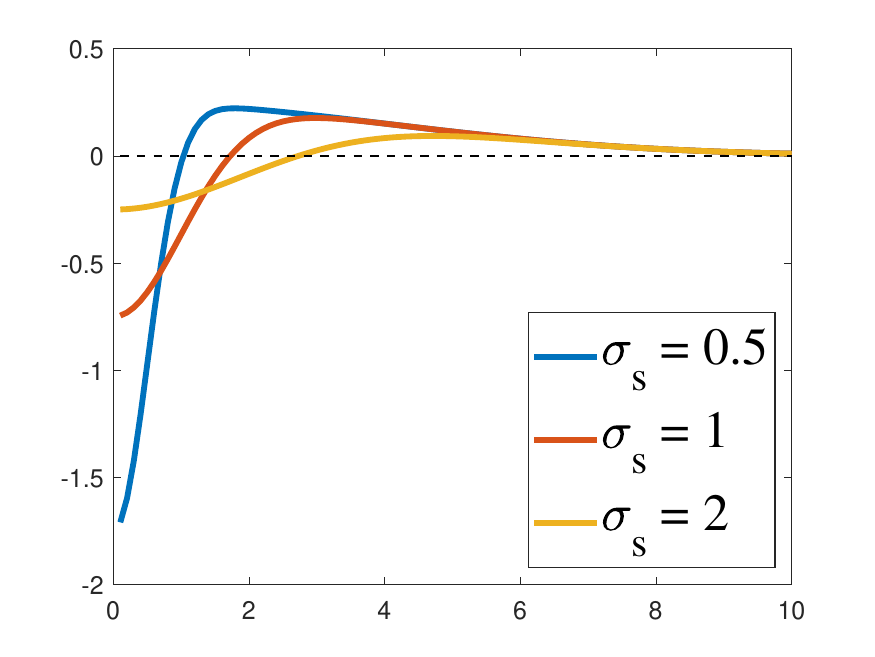}
		\caption{$e_{\sigma_0} -e^{s}_{\sigma_s}$.}
		\label{Fig.sub.2}
	\end{subfigure}
	\begin{subfigure}[t]{0.24\textwidth}
		\centering
		\includegraphics[width= 1.1\textwidth]{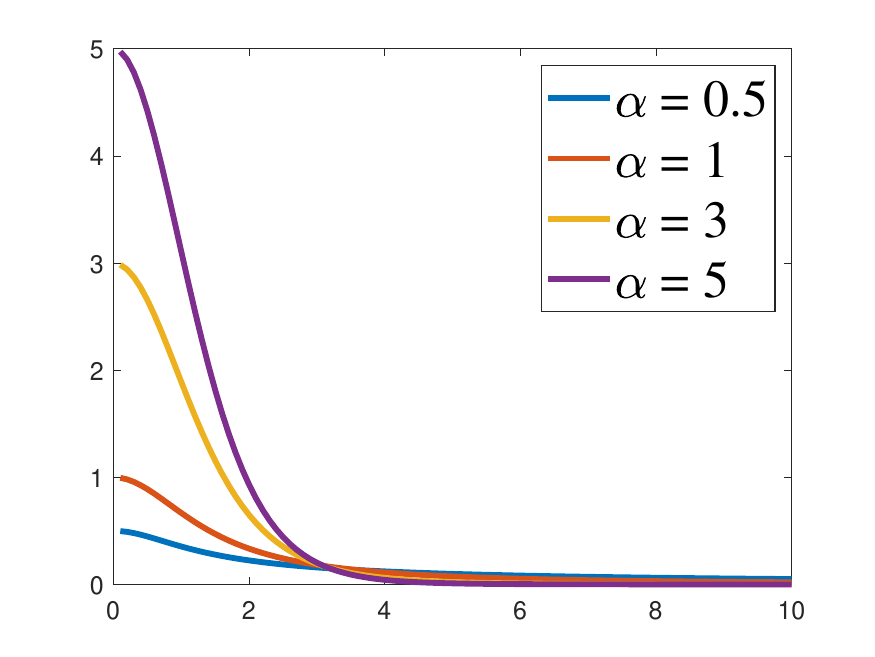}
		\caption{$\alpha (1 + \frac{r^2}{2\alpha})^{-\alpha}$.}
		\label{Fig.sub.3}
	\end{subfigure}	
	\begin{subfigure}[t]{0.24\textwidth}
		\centering
		\includegraphics[width= 1.1\textwidth]{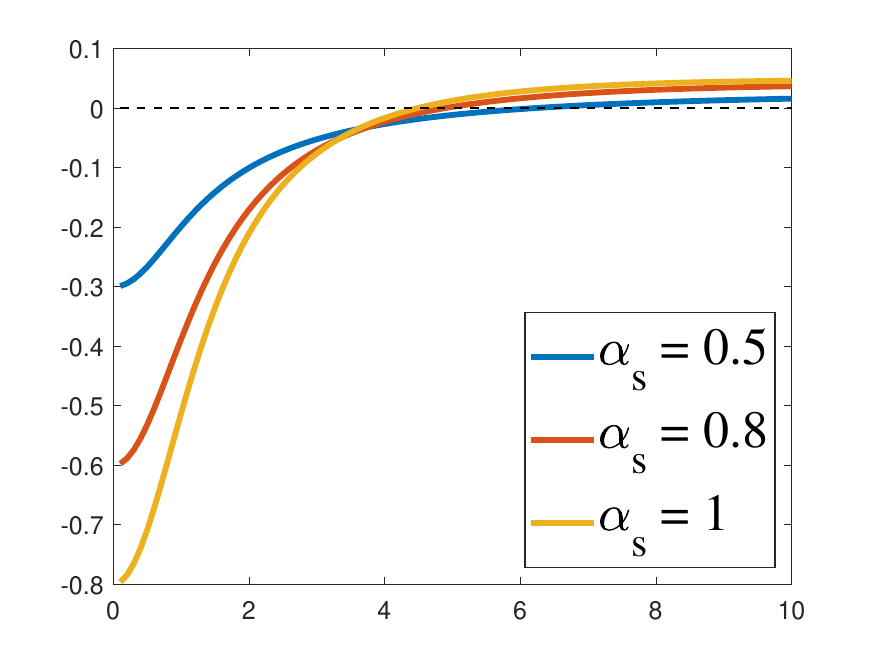}
		\caption{$e_{\alpha_{0}} -e^{s}_{\alpha_s}$.}
		\label{Fig.sub.4}
	\end{subfigure}	
\caption{
(a) Rescale Gaussian RBF kernel $\{e_{\sigma_i}(r)\}$, where $\sigma_i \in \{1,2,4,5\}$; (b) shows (a) with stabilizing term $\widetilde e_\sigma(r) = e_{\sigma_0}(r) - e_{\sigma_{s_i}}(r)$ where $\sigma_0=4$ and $\sigma_{s_i} \in \{0.5,1,2\}$; (c) rescale rational quadratic kernel $\{e_{\alpha_i}(r)\}$, where $\alpha_i \in \{0.5,1,3,5\}$; (d) shows (c) with stabilizing term $\widetilde e_\alpha(r) = e_{\alpha_{0}}(r) - e_{\alpha_{s_i}}(r)$ where $\alpha_0 = 0.2$ and $\alpha_{s_i} \in \{0.5,0.8,1\}$.
The dashed line represents zero values for comparison.}
\label{Fig.particle-based distance}
\end{figure}

\FloatBarrier 
\paragraph{Choice of $s(\bm{x},\bm{y})$.} First, we propose a rescaling distance $e_{k}(\bm{x},\bm{y})$ parameterized with a scaler $k$:
\begin{itemize}
    \item Rescale Gaussian RBF kernel (Fig. \ref{Fig.sub.1}):
\begin{equation}
e_{\sigma}(\bm{x},\bm{y}) =\frac{1}{\sigma} \exp \left(-\frac{1}{2\sigma^2}\left\|\bm{x}-\bm{y}\right\|^2\right).
\label{Eqn: rescale MMD Gua Kernel}
\end{equation}
    \item Rescale rational quadratic kernel (Fig. \ref{Fig.sub.3}):
\begin{equation}
e_{\alpha}(\bm{x},\bm{y}) = \alpha (1 + \frac{\|\bm{x} - \bm{y}\|^2}{2\alpha})^{-\alpha}.
\label{Eqn: rescale MMD Ratio kernel}
\end{equation}
    \item Elastic interaction term:
    \begin{equation}
e_{m}(\bm{x},\bm{y}) = \frac{1}{\|\bm{x}-\bm{y}\|^{m}}.
\end{equation}
\end{itemize}

In EIEG GAN \cite{refEIEGGAN}, $s(\bm{x},\bm{y}) = e_{m}(\bm{x},\bm{y})$ is a higher order term.
Specifically, the stabilized distance is $ \widetilde e_m(\bm{x},\bm{y}) = \frac{1}{r^{n-1}} - \epsilon \frac{1}{r^m}$ where the stabilizing term is $e_m(\bm{x},\bm{y}) = \frac{1}{r^m}$ with $m>n-1$. 
Such form is consistent with the he Lennard-Jones potential \cite{refLJpotential} $V_{LJ} = 4\epsilon[(\frac{\sigma}{r})^{12} - (\frac{\sigma}{r})^6] $
in molecular dynamics. 
Here we can also propose a similar stabilizing term for MMD GANs.
The stabilized distance for MMD GAN with Gaussian RBF kernel \cite{refMMDGAN1} is (Fig. \ref{Fig.sub.2}),
\begin{equation}
\widetilde e_\sigma(\bm{x},\bm{y}) = e_{\sigma_1}(\bm{x},\bm{y}) - \epsilon e_{\sigma_2}(\bm{x},\bm{y}).
\end{equation}
where $\sigma_2 < \sigma_1 $, and $\varepsilon > e^{\frac{|\bm{\xi}|(\sigma_2^2 - \sigma_1^2 )}{4}}$ such that $\mathcal{F}(\frac{1}{\sigma_1}e^{-\frac{\bm{x}^2}{2\sigma_1^2}} - \varepsilon \frac{1}{\sigma_2}e^{-\frac{\bm{x}^2}{2\sigma_2^2}}) < 0$ for all Fourier mode $\bm{\xi}$, 
and thus the training of discriminator becomes stable in our framework.
The stabilized distance for MMD GAN with rational quadratic kernel \cite{refMMDGAN2} is (Fig. \ref{Fig.sub.4}),
\begin{equation}
\widetilde e_\alpha(\bm{x},\bm{y}) = e_{\alpha_1}(\bm{x},\bm{y}) - \epsilon e_{\alpha_2}(\bm{x},\bm{y}),
\end{equation}
where $\alpha_2 > \alpha_1 $. 
The setting of $\alpha_1, \alpha_2$ is more complex for the rational kernel, and we will provide more discussions in the Appendix.

\textbf{Parameter $\epsilon$.} The selection of the parameter $\epsilon$ is critical for the success of the stabilizing approach. On the one hand, $\epsilon$ should be large enough to stabilize the training by ensuring that $\mathcal{F}(e(\|\bm{x} \|)- \epsilon s(\|\bm{x} \|))<0$. For example, in MMD GAN with a Gaussian RBF kernel, we found that $\epsilon \geq 1$ is sufficient to stabilize the training. On the other hand, $\epsilon$ cannot be too large, as this would cause the data points from the same distribution to be too scattered in the feature space, and would also reduce the adversarial nature of the discriminator.

To understand the effect of the stabilizing term on the training process, we use a molecular dynamics analogy to interpret the optimization of the discriminator loss with the stabilizing term, i.e., $\max_{D} \mathcal{L}^{s}$. In this analogy, we consider the force between two samples in the feature space, where $\widetilde{e}(\bm{x},\bm{y})$ represents the potential energy between them. If $\widetilde{e}(\bm{x},\bm{y}) < 0$, this indicates that the force between the two particles is repulsive, while if $\widetilde{e}(\bm{x},\bm{y}) > 0$, the force is attractive. As shown in Fig. \ref{Fig.sub.2} and Fig. \ref{Fig.sub.4}, when two samples are close to each other, the force between them is repulsive, while when they are far apart, the force between them is attractive. Therefore, if the stabilizing term $\epsilon$ is set too large, it will cause too much repulsion between samples from the same distribution, resulting in the samples being spread too thinly in the feature space. On the other hand, if $\epsilon$ is too small, it may lead to training instability and mode collapse, as the samples in the feature space collapse. More discussion can be found in Appendix.

\section{Related Works} \label{section5}

\paragraph{MMD GAN related work.}
In the original MMD GAN \cite{refMMDGAN1}, the discriminator is viewed as a kernel selection mechanism.
Here, we propose an alternative perspective that the discriminator can be regarded as a feature transformation mapping. 
This view provides insights into various approaches to improve MMD GAN performance by preserving more information about the data and samples in the feature space. 
For example, in \cite{refMMDGANimprove}, the proposed repulsive discriminator loss can be understood from our perspective as preventing sample collapse in feature space. In \cite{refCRGAN}, the addition of consistency regularization to the discriminator loss can be understood as grouping similar samples closely in the feature space.
Furthermore, we utilize this perspective to analyze the training stability of MMD GAN via Wasserstein gradient flow. Our results suggest that MMD GAN training is unstable. 
This finding is consistent with some experimental results in \cite{refMMDGAN2}. 
Our approach is simpler and more accessible than previous theoretical works, such as \cite{refMMDflow}, which analyze the convergence of MMD GAN through gradient flow. Additionally, to the best of our knowledge, our work is the first to perform training stability analysis on MMD GANs.

\paragraph{Stabilization methods for GANs.} 
Training stability is a critical issue in GANs, and various methods have been proposed to address this challenge \cite{refPreWGAN,refWGAN,refWGANGP,refstaGAN1,refstaGAN2,refCRGAN}. One common approach involves imposing Lipschitz conditional restrictions on the discriminator through normalization and regularization techniques. Normalization methods such as spectral normalization \cite{refSNGAN} and gradient normalization \cite{refGNGAN} have been effective in stabilizing training. Regularization methods, such as adding a gradient penalty to the discriminator loss \cite{refWGANGP,refGANregu1,refGANregu2,refGANregu3,refMMDGANGP}, have also been widely adopted. Using our analysis framework, we analyze the impact of adding a gradient penalty on training stability and find that it does indeed stabilize training (see Appendix). When a gradient penalty is added to the discriminator's loss as a stabilizing term, it appears in the gradient flow as an additional Laplacian term. However, this can cause the discriminator to become overly smooth, and the generated samples may become connected, leading to mode collapse, while the proposed stabilizing term has no such problem. Although spectral normalization is an effective method for stabilizing GAN training, it has been reported that it may lead to mode collapse in SNGAN \cite{refSNGANmodecollaspe}. Our stabilizing term, which creates a repulsive force, can prevent sample points from collapsing together, thereby addressing the mode collapse issue. More discussion can be found in the Appendix.

\FloatBarrier

\begin{figure}[!hbtp]
        \centering 
	\begin{subfigure}[t]{0.24\textwidth}
		\centering
		\includegraphics[width= 1\textwidth]{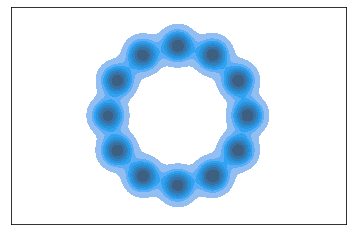}
		\caption{True Density.}
		\label{Fig.sub.2a}
	\end{subfigure}
	\begin{subfigure}[t]{0.24\textwidth}
		\centering
		\includegraphics[width= 1\textwidth]{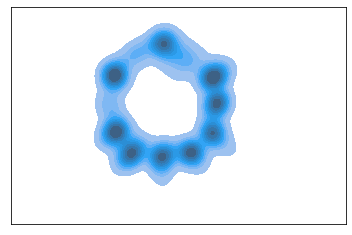}
		\caption{MMD GAN}
		\label{Fig.sub.2b}
	\end{subfigure}
	\begin{subfigure}[t]{0.24\textwidth}
		\centering
		\includegraphics[width= 1\textwidth]{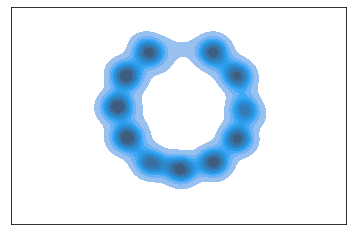}
		\caption{MMD GAN-GP.}
		\label{Fig.sub.2c}
	\end{subfigure}	
	\begin{subfigure}[t]{0.24\textwidth}
		\centering
		\includegraphics[width= 1\textwidth]{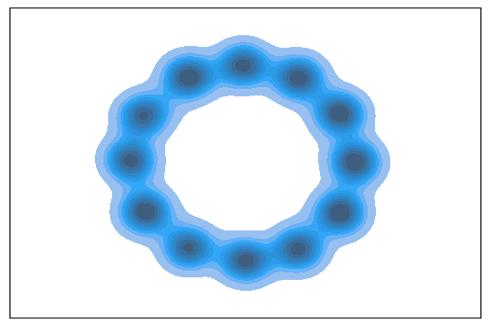}
		\caption{Stabilized  MMD GAN.}
		\label{Fig.sub.2d}
	\end{subfigure}	
 \begin{subfigure}[t]{0.24\textwidth}
		\centering
		\includegraphics[width= 1\textwidth]{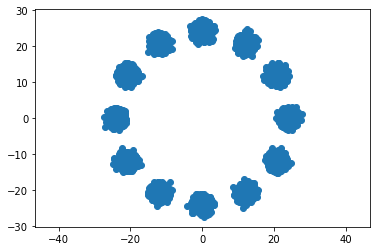}
		\caption{True Samples.}
		\label{Fig.sub.2e}
	\end{subfigure}	
  \begin{subfigure}[t]{0.24\textwidth}
		\centering
		\includegraphics[width= 1\textwidth]{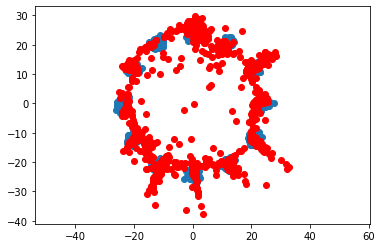}
		\caption{MMD GAN.}
		\label{Fig.sub.2f}
	\end{subfigure}	
   \begin{subfigure}[t]{0.24\textwidth}
		\centering
		\includegraphics[width= 1\textwidth]{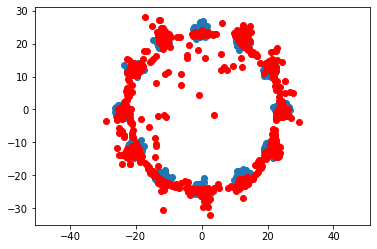}
		\caption{MMD GAN-GP.}
		\label{Fig.sub.2g}
	\end{subfigure}	
   \begin{subfigure}[t]{0.24\textwidth}
		\centering
		\includegraphics[width= 1\textwidth]{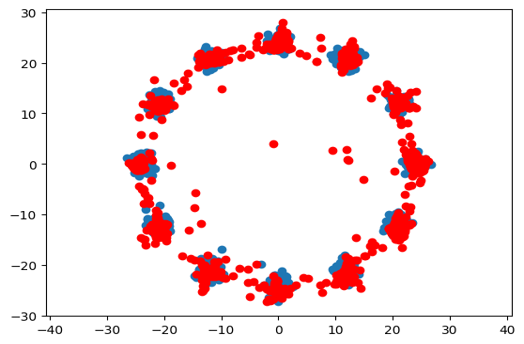}
		\caption{Stabilized MMD GAN.}
		\label{Fig.sub.2h}
	\end{subfigure}	
\caption{MMD GAN variants with Gasussian RBF kernel for Gaussian mixture sampling. First row: KDE plots for generated samples. Second row: blue points represent samples from the dataset, and red points represent the generated samples. }
\label{Fig.MoG}
\end{figure}

\section{Experiments} 
\label{section6}
To validate both the proposed analysis and stabilizing method, we take an example of MMD GAN with a Gaussian RBF kernel $k_{\sigma}^{rbf}\left(\bm{x}, \bm{y}\right)$,
and conduct experiments on synthetic and real datasets (CIFAR-10 \cite{refCIFAR-10}). More experiments and detailed settings are provided in Appendix.

\paragraph{Gaussian Mixture.}
We conduct Gaussian mixture experiments to compare our method with the original MMD GAN \cite{refMMDGAN1} and MMD GAN-GP \cite{refMMDGANGP}. 
We sample from a mixture of eight two-dimensional Gaussian distributions, and all models are trained with 2000 particles. 
The results are shown in Fig. \ref{Fig.MoG}: (1) The generated samples from MMD GAN are disorganized (Fig. \ref{Fig.sub.2f}), which is caused by the instability of the training process. (2) Mode collapse occurs in MMD GAN-GP where the generator fails to grasp all the modes of the distribution, as shown in Fig. \ref{Fig.sub.2c}. Also in Fig. \ref{Fig.sub.2g}, the generated samples all link together. This is because the gradient penalty added to the discriminator loss as a stabilizer makes the generated sample points more scattered in the feature space. (3) As shown in Fig. \ref{Fig.sub.2h}, our proposed method successfully grasps all the modes of the Mixture Gassuian.

\begin{wrapfigure}{r}{0.4\textwidth}
  \begin{center}
  \includegraphics[width=0.45 \textwidth]{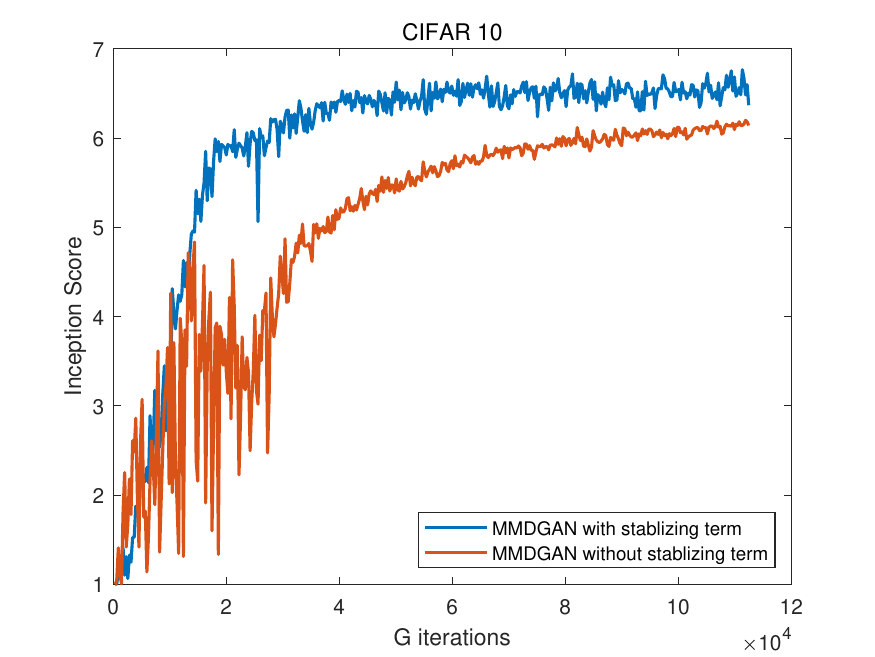}
  \end{center}
  \caption{Training Curves on CIFAR10 of MMD GAN  with or 
  without stabilizing term.}
  \label{Fig:MMD GAN stability}
\end{wrapfigure}

\paragraph{Image Generation.}
To verify the results of our stability analysis presented in Table \ref{Table 1}, which showed that the training dynamics of MMD GAN with a Gaussian RBF kernel is unstable, and to demonstrate the effectiveness of our approach, we conduct image generation experiments on the CIFAR-10 dataset. 
We use the same network architecture and hyperparameters as in the original MMD GAN paper \cite{refMMDGAN1}.
We then use a linear combination of particle-based distances with different scales, i.e., $e_{rbf}(\bm{x},\bm{y}) = \sum_{i = 1}^{K}e_{\sigma_i}(\bm{x},\bm{y})$ as in \cite{refMMDGAN1},
where $K = 4$ and $\sigma_{i} = \{2,4,8,16\}$. 
For the stabilizing term, we set $s(\bm{x},\bm{y}) = \sum_{i = 1}^{K}e_{\sigma_i}(\bm{x},\bm{y})$, where $\sigma_{i} = \{1,\sqrt{2},2,2\sqrt{2}\}$. 
To provide a more intuitive representation of the stability of GAN training for image generation, we use the Inception score \cite{refInception} to plot the training curve. 
The results are shown in Fig. \ref{Fig:MMD GAN stability}. 
As indicated in the figure, the original MMD GAN suffers from training instability, while our stabilizing term significantly improves the stability of training and enhances the quality of the generated images.

\section{Conclusion and Discussion}
This study introduces a novel framework for analyzing the training stability of particle-based distance GANs using the Wasserstein gradient flow. 
We use the proposed perturbation evolution dynamics to analyze the training stability and our analysis reveals that the training of these GANs is unstable. 
Moreover, we develop a new stabilizing method by introducing a stabilizing term in the loss function of the unstable network. The empirical results 
validate our analysis and demonstrate the effectiveness of the proposed stabilizing method.

Our analysis in this paper focuses on particle-based distance GANs. 
A property of those GANs is that the probability density function of the samples in feature space $p_{f_{\mathrm{data}}}$, is smooth, enabling us to use the Wasserstein gradient flow to analyze their evolution stability. 
For Vanilla GAN \cite{refGAN}, the probability density function in feature space of the generated samples is discrete. 
Therefore, the Wasserstein gradient flow framework,
which describes the evolution of a smooth probability density function, cannot apply to Vanilla GAN. 
Alternatively, we can derive the perturbation evolution dynamics based on the functional gradient flow to analyze the training stability of the Vanilla GAN's generator and discriminator training, i.e., $\frac{\partial G}{\partial t} = -\frac{\delta \mathcal{L}{G}}{\delta G}$ and $\frac{\partial D}{\partial t} = \frac{\delta \mathcal{L}{D}}{\delta D}$. 
We provide this analysis in the Appendix.
Additionally, our framework can be extended to the case where we take account of the neural network architectures and analyze the perturbation evolution equation through the gradient flow of the network parameters. By finding the perturbation evolution equation for the corresponding gradient flow, our framework can be extended to various training stability analysis.

\section*{Acknowledgements}

The work of Y.X. was supported by the Project of Hetao Shenzhen-HKUST Innovation Cooperation Zone HZQB-KCZYB-2020083.

{\small
\bibliographystyle{plain}
\bibliography{main}
}


\newpage

\begin{appendices}

{\Large\textbf{Appendix}}

\section{Physical Interpretation of Proposed Framework} \label{Appendix A}

Section \ref{section3.2} describes our proposed stability analysis framework, which is based on the perspective that the discriminator can be viewed as a feature transformation mapping. Figure \ref{Fig: framework} provides an intuitive illustration of this perspective.

\begin{figure}[!hbtp]
		\centering
		\includegraphics[width= 0.95\textwidth]{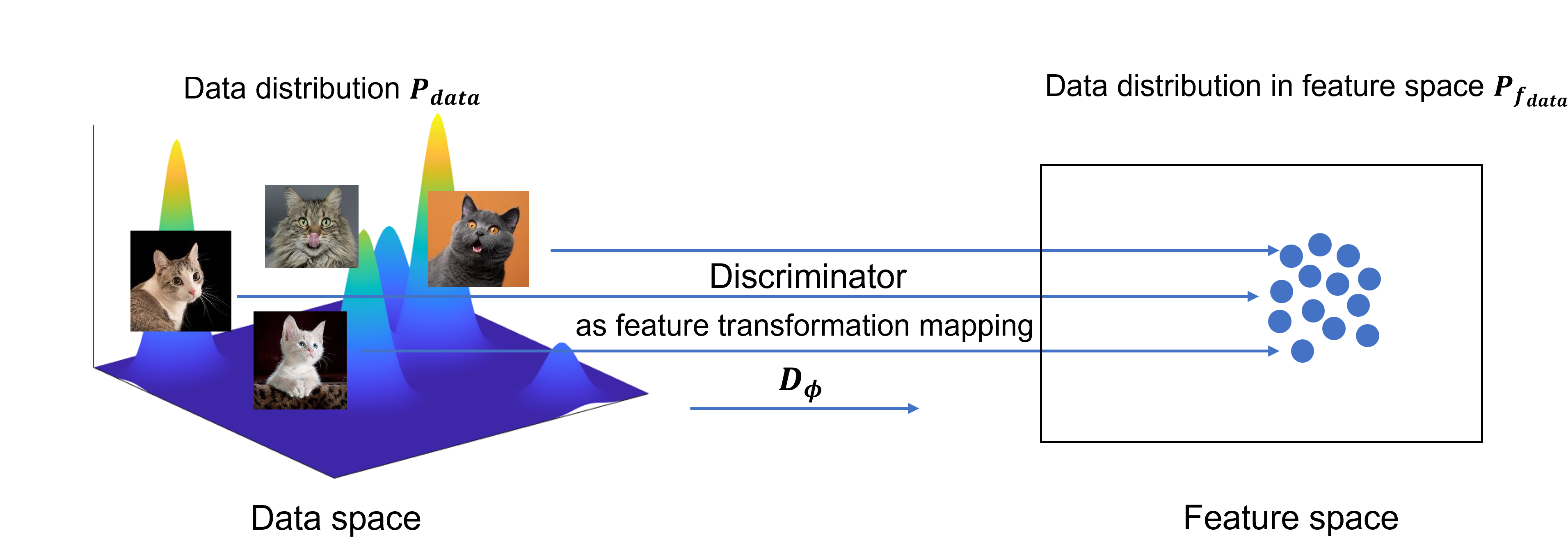}
\caption{Ilustration of our proposed framework which presents an alternative perspective for viewing the discriminator in particle-based GANs as a feature transformation mapping. The blue points stand for the corresponding data samples in feature space. The data samples (cat images) in the figure are from the Internet.}
\label{Fig: framework}

\end{figure}	
In our framework, we analyze the training stability of particle-based GANs through the evolution equation of generated samples in feature space. We offer a physical interpretation of the training process of the generator and discriminator. Fig.~\ref{Fig.md.sub.1} illustrates that in the training process of the generator $\min_G \mathcal{L}_G$ (Eq.~\eqref{Eqn. loss of GEN}), the generated samples experience a repulsive force between each other while the force between generated and real data samples is attractive. Fig.~\ref{Fig.md.sub.2} shows that in the discriminator training $\max_D \mathcal{L}_D$ (Eq.~\eqref{Eqn. Loss of Dis}), the force between the generated samples is attractive while the force between the generated and real data samples is repulsive. 
Fig.~\ref{Fig.md.sub.3} demonstrates that during the training of a stabilizing discriminator $\max_D \mathcal{L}_D^s$ (Eq.~\eqref{Eqn Loss of sta D}), the force between the generated and real data samples is repulsive, and if two generated samples are close to each other, the force between them is also repulsive, otherwise it is attractive. The stability effect for adding the stabilizing term can be found in the main paper.

\begin{figure}[!hbtp]
	\begin{subfigure}[t]{0.33\textwidth}
		\centering
		\includegraphics[width= 1\textwidth]{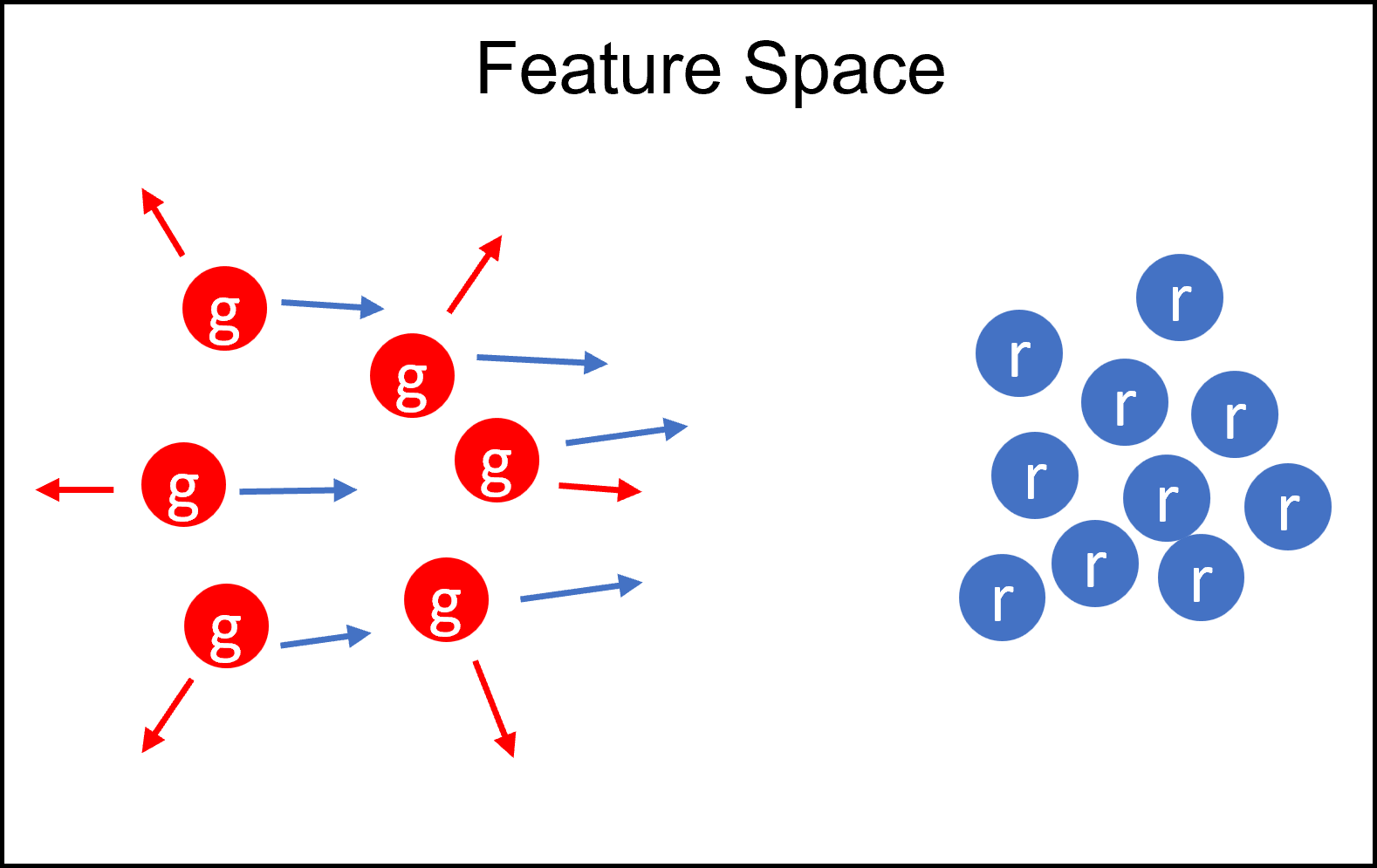}
		\caption{$\min_G \mathcal{L}_G$  (Eq.(\ref{Eqn: Gen particle flow})).}
		\label{Fig.md.sub.1}
	\end{subfigure}
	\begin{subfigure}[t]{0.33\textwidth}
		\centering
		\includegraphics[width= 1\textwidth]{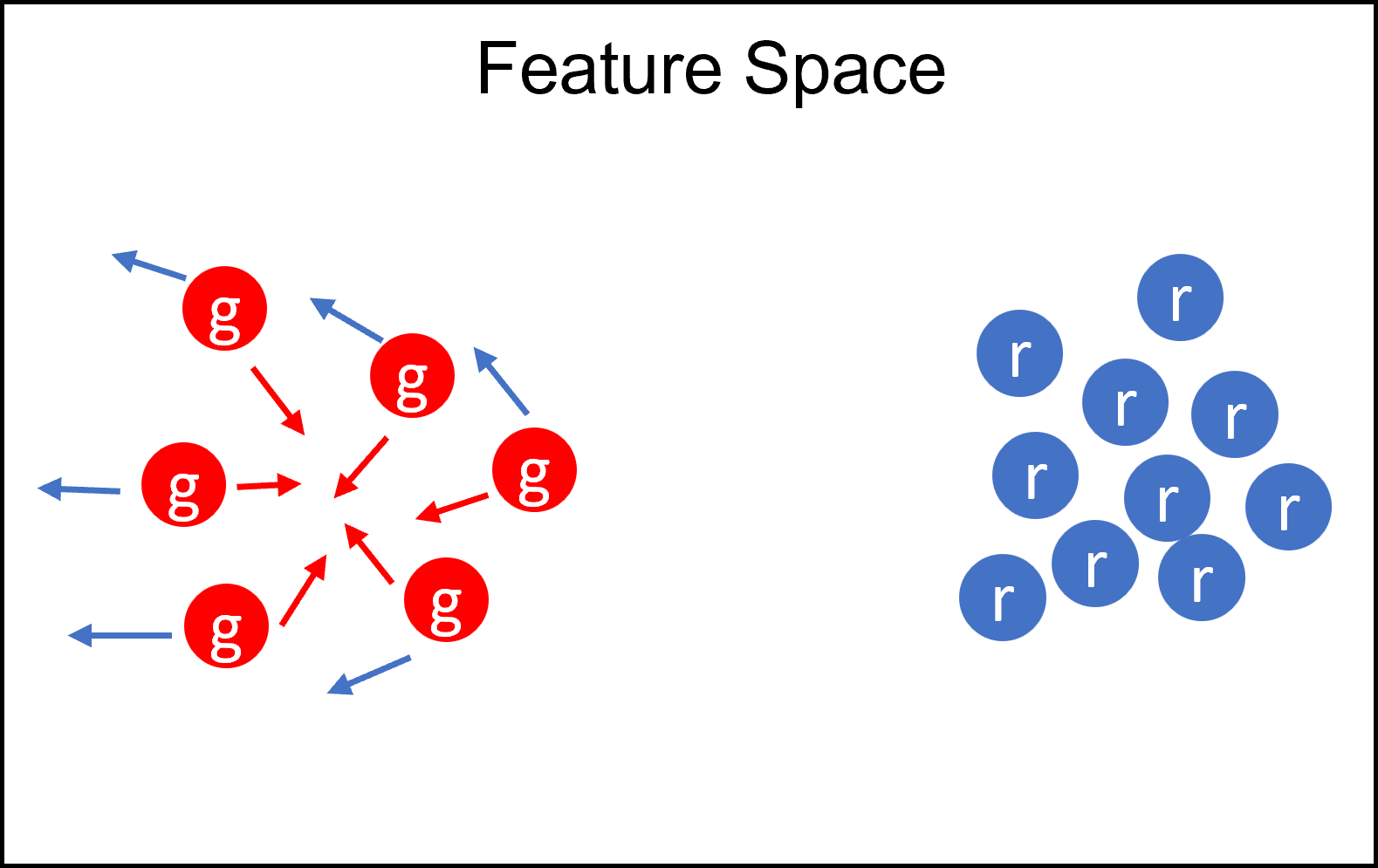}
		\caption{$\max_D \mathcal{L}_D$  (Eq.(\ref{Eqn: Discri particle flow})).}
		\label{Fig.md.sub.2}
	\end{subfigure}
	\begin{subfigure}[t]{0.33\textwidth}
		\centering
		\includegraphics[width= 1\textwidth]{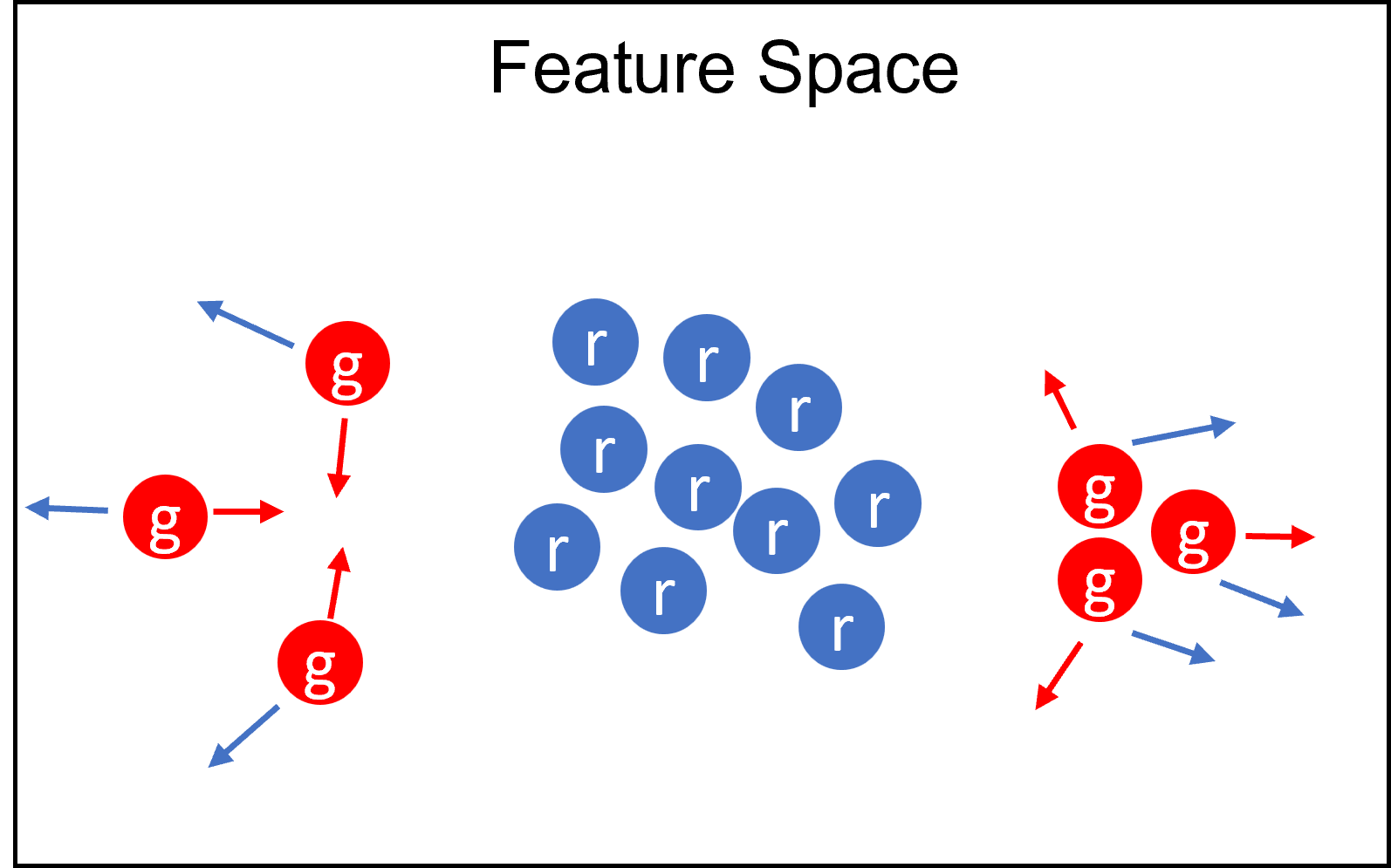}
		\caption{$\max_D \mathcal{L}^{s}_D$ \ref{Eqn Loss of sta D} .}
		\label{Fig.md.sub.3}
	\end{subfigure}	
        
\caption{Ilustration of the training process of the particle-based GANs. Red particles with the letter g stand for the generated samples in feature space and blue particles with the letter r stand for the real data samples in feature space. The red arrows depict the force exerted by the generated samples, while the blue arrows represent the force exerted by the real data samples.}
\label{Fig. Molecular Dynamics}
\end{figure}

\newpage

\section{Proofs}
Throughout this section, we analyze the stability in the simplest case where the perturbation function is added to a constant-valued density with respect to space where the constant value may change with time.

\subsection{Stability analysis for particle-based distance GANs}

In this section, we demonstrate how to derive the perturbation evolution equation (Eq.~\eqref{Eqn: perturbation general}) of the Wasserstein gradient flow of particle-based distance GANs. We give detailed proofs for the analytical results presented in Table \ref{Table 1} in section \ref{section3.3}.

\subsubsection{Derivation of perturbation evolution equation}

\begin{proposition}
The perturbation evolution equation of the Wasserstein gradient flow of particle-based distance GANs in Fourier space is 
$$
        \frac{d \hat{v}}{dt} = \mp C|\bm{\xi}|^{2}\hat{v}\mathcal{F}(e(\|\bm{x}\|)).
$$
The constant $C \geq 0$ is associated with $p_{f_{\mathrm{data}}}$. In the context of GANs, the negative sign indicates the training dynamics of the generator, while the positive sign indicates the training dynamics of the discriminator. The function $e(\bm{x},\bm{y}) = e(\|\bm{x}-\bm{y}\|)$ and $\mathcal{F}(e(\|\bm{x}\|))$ denotes the Fourier transform of $e(\|\bm{x}\|)$
\footnote{Here we define the Fourier transform of $f(\bm{x})$ as $\mathcal{F}(f(\bm{x}))(\bm{\xi})=\int_{\mathbb{R}^n}f(\bm{x}) e^{-i 2 \pi(\bm{\xi}\cdot \bm{x})} d\bm{x}$}.
We use $\hat{v}$ to denote the Fourier transform of $v$, and $\bm{\xi}$ to represent the Fourier mode. 
\end{proposition}

\begin{proof}
Consider generated samples $G(\bm{z})$ in the data space, with a fixed discriminator $D$, and denote the generated samples in feature space as $X_t := D(G(\bm{z})) $ .
The evolution of $X_t$ can be described by
\begin{equation}
dX_t = \left[ - 2\mathbb{E}_{\bm{y} \sim \mathbb{P}_{f_{\mathrm{data}}}} (\nabla e(X_t,\bm{y})) + 2\mathbb{E}_{\bm{y} \sim \mathbb{P}_{f_{g}}} (\nabla e(X_t,\bm{y})) \right]dt = -\nabla (\frac{\delta E}{\delta p_{f_{g}}}),   \quad X_0 \sim \mathbb{P}_{f_{\mathcal{N}(I, \bm{0})}}.
\end{equation}
The corresponding density flow is 
\begin{equation}
\begin{aligned}
      \frac{\partial p_{f_{g}} }{\partial t}  
      & =
      -\nabla \cdot\left[ p_{f_g} \left(-\nabla \frac{\delta E}{\delta p_{f_g}} \right) \right] 
      =  
      \nabla \cdot\left(p_{f_g} \nabla \frac{\delta E}{\delta p_{f_g}}\right)  \\
      &=
      \nabla \cdot\left[
      p_{f_g} \nabla \bigg(2\mathbb{E}_{\bm{y} \sim \mathbb{P}_{f_{g}}}e(\bm{x}, \bm{y})- 2\mathbb{E}_{\bm{y} \sim \mathbb{P}_{f_{\mathrm{data}}}}e(\bm{x}, \bm{y})\bigg)\right]\\
      &=
      2\nabla \cdot\left[ p_{f_g} \nabla\bigg( \int_{\mathbb{R}^d}(e(\bm{x}, \bm{y})p_{f_g}(\bm{y})- e(\bm{x}, \bm{y})p_{f_{\mathrm{data}}}(\bm{y}))d\bm{y}\bigg)\right],
\end{aligned}
\label{Eqn: Wgflow}
\end{equation}
where $p_{f_{g}}$ is the constant-valued density of the generated samples $X_t$ in the feature space.
Assume we add a small perturbation $v, |v| \ll 1$ 
to the density $ p_{f_{g}}$
and we denote $C_0:=p_{f_{g}} \geq 0$, which may various with time.
Substituting the perturbed density $C_0 + v$ into the Wasserstein gradient flow (Eqn.~\eqref{Eqn: Wgflow}) and only keeping the linear term of $v$ since $|v| \ll 1$, we can obtain
\begin{equation}
    \frac{\partial v}{\partial t} = C_0 \Delta  \int_{\mathbb{R}^d}e(\bm{x},\bm{y})vd\bm{y}= C_0 \Delta  \int_{\mathbb{R}^d}e(\|\bm{x}-\bm{y}\|)vd\bm{y} =  C_0 \Delta  (e(\|\bm{x}\|) * v)(\bm{x}).
\end{equation}

Taking Fourier transform on both sides of the equation
\begin{equation}
    \frac{d \hat{v}}{dt} = - 2C_0 (2\pi)^2|\bm{\xi}|^{2}\hat{v}\mathcal{F}(e(\|\bm{x}\|)) = - C|\bm{\xi}|^{2}\hat{v}\mathcal{F}(e(\|\bm{x}\|)),
\end{equation}
where $\hat{v}$ is the Fourier transform of $v$, $\bm{\xi}$ is the Fourier mode, and $C = 2(2\pi)^2C_0$. 
Thus in the Fourier space, the solution for the perturbation term $\hat{v}$ is
\begin{equation}
    \hat{v} = \hat{v}_0e^{-C|\bm{\xi}|^2 \mathcal{F}(e(\|\bm{x}\|)) t},
\end{equation}
where $\hat{v}_0$ is the initial value of $\hat{v}$. If $\mathcal{F}(e(\|\bm{x}\|))>0$ for all $\bm{\xi}$, then the perturbation with all $|\bm{\xi}|$ decay as training processes, i.e., $|\hat{v}| = |\hat{v}_0|e^{-C|\bm{\xi}|^2 \mathcal{F}(e(\|\bm{x}\|)) t} \rightarrow 0$ as $t \rightarrow \infty$, the training of generator is stable; 
if $\mathcal{F}(e(\|\bm{x}\|))<0$ for all $\bm{\xi}$, then the perturbation with all  $|\bm{\xi}|$ grows as training processes,  i.e., $|\hat{v}| = |\hat{v}_0|e^{-C|\bm{\xi}|^2 \mathcal{F}(e(\|\bm{x}\|)) t} \rightarrow \infty$ as $t \rightarrow \infty$; if the sign of $\mathcal{F}(e(\|\bm{x}\|))(\bm{\xi})$ depends on $|\bm{\xi}|$, both the generator is unstable for the value of $\bm{\xi}$, where $\mathcal{F}(e(\|\bm{x}\|))(\bm{\xi})<0$.

With a fixed generator $G$, the evolution of the generated samples in feature space $X_t := D(G(\bm{z}))$ is 

\begin{equation}
dX_t = \left[2\mathbb{E}_{\bm{y} \sim \mathbb{P}_{f_{\mathrm{data}}}} (\nabla e(X_t,\bm{y})) - 2\mathbb{E}_{\bm{y} \sim \mathbb{P}_{f_{g}}} (\nabla e(X_t,\bm{y})) \right]dt = \nabla (\frac{\delta E}{\delta p_{f_{g}}}),   \quad X_0 \sim \mathbb{P}_{f_{\mathcal{N}(I, \bm{0})}}.
\end{equation}

Thus the corresponding density flow for the generated samples in feature space $X_t$ is 

\begin{equation}
\begin{aligned}
      \frac{\partial p_{f_{g}} }{\partial t}  &= -\nabla \cdot\left(p_{f_g} (\nabla \frac{\delta E}{\delta p_{f_g}})\right) \\
      &=\nabla \cdot\left(p_{f_g} \nabla \bigg(-2\mathbb{E}_{\bm{y} \sim \mathbb{P}_{f_{g}}}e(\bm{x}, \bm{y})+ 2\mathbb{E}_{\bm{y} \sim \mathbb{P}_{f_{\mathrm{data}}}}e(\bm{x}, \bm{y})\bigg)\right) \\
      &= -2\nabla \cdot\left(p_{f_g} \nabla\bigg( \int_{\mathbb{R}^d}(e(\bm{x}, \bm{y})p_{f_g}(\bm{y})- e(\bm{x}, \bm{y})p_{f_{\mathrm{data}}}(\bm{y}))d\bm{y}\bigg)\right) .
\end{aligned}
\end{equation}

Thus the perturbation evolution equation is
\begin{equation}
    \frac{\partial v}{\partial t} = -2C_0 \Delta  \int_{\mathbb{R}^d}e(\bm{x},\bm{y})vd\bm{y}= -2C_0 \Delta  \int_{\mathbb{R}^d}e(\|\bm{x}-\bm{y}\|)vd\bm{y} =  -2C_0 \Delta  (e(\|\bm{x}\|) * v)(\bm{x}).
\end{equation}

Taking Fourier transform on both sides of the equation
\begin{equation}
    \frac{d \hat{v}}{dt} =  2C_0 (2\pi)^2|\bm{\xi}|^{2}\hat{v}\mathcal{F}(e(\|\bm{x}\|)) =  C|\bm{\xi}|^{2}\hat{v}\mathcal{F}(e(\|\bm{x}\|)).
\end{equation}

Worth noticing that, the only difference between the perturbation evolution equation of generator and discriminator is the sign before the formula, one is positive the other is negative. This is caused by the minmax formulation $\max_D \min_G E(G, D)$.

In conclusion, the evolution equation for the perturbation $v$ in Fourier spaces always takes the form 
\begin{equation}
        \frac{d \hat{v}}{dt} = \mp C|\bm{\xi}|^{2}\hat{v}\mathcal{F}(e(\|\bm{x}\|)).
\end{equation}

\end{proof}

\textbf{Remark.} Here, we assume that $p_{f_{\mathrm{data}}}= C_0$. The stability and instability are local effects. In this case where $p_{f_{\mathrm{data}}}$ is not constant, $p_{f_{\mathrm{data}}}$ can still be approximated as a constant locally, which allows the above analysis to be applied.

From the above proposition, if $\mathcal{F}(e(\|\bm{x}\|))(\bm{\xi}) > 0$ for all $\bm{\xi}$, the training dynamics of the generator is stable, with $\hat{v} \rightarrow 0$ as $t \rightarrow \infty$, and the corresponding training of the discriminator is unstable. Conversely, if $\mathcal{F}(e(\|\bm{x}\|))(\bm{\xi}) < 0$ for all $\bm{\xi}$, the situation is reversed, with unstable generator training and stable discriminator training. If the sign of $\mathcal{F}(e(\|\bm{x}\|))(\bm{\xi})$ depends on $|\bm{\xi}|$, both the generator and discriminator training are unstable for some value of $\bm{\xi}$.

\subsubsection{Stability analysis for Cramér GAN}
\begin{proposition}
In  Cramér GAN, the training process of the generator, i.e.,$\min_G \mathcal{L}_G$, is stable, while the training process of the discriminator, i.e., $\max_D \mathcal{L}_D$, is unstable.
\end{proposition}

\begin{proof}
   The objective function for Cramér GAN is
    \begin{equation}
\begin{aligned}
E(G,D) = &
-2 \mathbb{E}_{\bm{x} \sim \mathbb{P}_{\mathrm{data}}, \bm{z} \sim \mathcal{N}(\bm{0},I)}[\|D(\bm{x})\| + \|D(G(\bm{z}))\| - \|D(\bm{x})-D(G(\bm{z}))\|]\\
&  +\mathbb{E}_{\bm{x}, \bm{x}^{\prime} \sim \mathbb{P}_{\mathrm{data}}}\left[\|D(\bm{x})\| + \|D(\bm{x}^{\prime})\| - \|D(\bm{x})-D(\bm{x}^{\prime}))\|\right]\\
&+\mathbb{E}_{\bm{z}, \bm{z}^{\prime} \sim \mathcal{N}(\bm{0},I)}\left[\|D(G(\bm{z}^{\prime}))\| + \|D(G(\bm{z}))\| - \|D(G(\bm{z}))-D(G(\bm{z}^{\prime}))\|\right]\\
=&-2 \mathbb{E}_{\bm{x} \sim \mathbb{P}_{\mathrm{data}}, \bm{z} \sim \mathcal{N}(\bm{0},I)}[ - \|D(\bm{x})-D(G(\bm{z}))\|]\\
&  +\mathbb{E}_{\bm{x}, \bm{x}^{\prime} \sim \mathbb{P}_{\mathrm{data}}}\left[ - \|D(\bm{x})-D(\bm{x}^{\prime}))\|\right]\\
&+\mathbb{E}_{\bm{z}, \bm{z}^{\prime} \sim \mathcal{N}(\bm{0},I)}\left[ - \|D(G(\bm{z}))-D(G(\bm{z}^{\prime}))\|\right].
\end{aligned}
\end{equation}

Thus in this case the actual particle-based distance is $e^{*}(x,y) = -\|\bm{x}-\bm{y}\|$. The Fourier transform for it is 
\begin{equation}
\mathcal{F}(-\|\bm{x}\|)= \frac{C_n}{|\bm{\xi}|^{n+1}},
\end{equation}

  where $n$ is the dimension of the feature space, $k=1,2,3...$ and $C_n >0$. For all $\bm{\xi}$,  $F(-\|\bm{x}\|) < 0$, from the above proposition, we know that for Cramér GAN, the training of generator is unstable while the training for discriminator is stable.
\end{proof}

\subsubsection{Stability analysis for MMD GAN with Gaussian RBF kernel}

\begin{proposition}
In MMD GAN with Gaussian RBF kernel, the training process of the generator, i.e., $\min_G \mathcal{L}_G$, is stable, while the training process of the discriminator, i.e., $\max_D \mathcal{L}_D$, is unstable.
\end{proposition}

\begin{proof}
  For MMD GAN with Gaussian RBF kernel, we have  
  $e(\bm{x},\bm{y}) = e^{-\frac{\|\bm{x}-\bm{y}\|^2}{2\sigma^2}}$. The Fourier transform for it is
  \begin{equation}    
  \mathcal{F}(e^{-\frac{\|\bm{x}\|^2}{2\sigma^2}})  = \sigma e^{-\frac{\sigma^2|\bm{\xi}|^2}{4}},
  \end{equation}
  where for all $\bm{\xi}$,  $\mathcal{F}(e^{-\frac{\|\bm{x}\|^2}{2\sigma^2}}) = \sigma e^{-\frac{\sigma^2|\bm{\xi}|^2}{4}}>0$. From the above proposition, we know that for MMD GAN with Gaussian RBF kernel, the training of generator is stable while the training for discriminator is unstable.

\end{proof}

\subsubsection{Stability analysis for MMD GAN with rational quadratic kernel}

\begin{proposition}
In MMD GAN with rational quadratic kernel, the stability of the training process depends on the value of $\alpha$.
\end{proposition}

\begin{proof}
        For MMD GAN with rational quadratic kernel, we have $e(\bm{x},\bm{y}) = (1 + \frac{\|\bm{x} - \bm{y}\|^2}{2\alpha})^{-\alpha}$. It should be noted that the Fourier transform of $(1 + \frac{|\bm{x}|^2}{2\alpha})^{-\alpha}$ does not have a uniform form that depends on $\alpha$. Here we only list some examples.
        \begin{itemize}
            \item For $\alpha = 1/2$ we have,
    \begin{equation}
    \mathcal{F}( (1 + \frac{\|\bm{x}\|^2}{2\alpha})^{-\alpha}) = \mathcal{F}(\frac{1}{2}\frac{1}{\sqrt{1+\|\bm{x}\|^2}}) = K_0(\xi),
    \end{equation}
    where $K_0(\xi)$ is the Bessel function and $K_0(\xi)>0$. From the above proposition, we know that for $\alpha = \frac{1}{2}$, the training of generator is stable while the training for discriminator is unstable.
    \item For $\alpha = 1$ we have, 
    \begin{equation}
    \mathcal{F}((1 + \frac{\|\bm{x}\|^2}{2\alpha})^{-\alpha}) = \mathcal{F}(\frac{1}{1+\frac{\|\bm{x}\|^2}{2}}) =  e^{|\bm{\xi}|},        
    \end{equation}
    where for all $\bm{\xi}$, $\frac{1}{\alpha} e^{|\bm{\xi}|}>0$. From the above proposition, we know that for $\alpha=1$, the training of generator is stable while the training for discriminator is unstable.
    
    \item For $\alpha = 2$ we have,
    \begin{equation}
              \mathcal{F}( (1 + \frac{\|\bm{x}\|^2}{2\alpha})^{-\alpha}) = \mathcal{F}(\frac{1}{(1+\frac{\|\bm{x}\|^2}{4})^2}) =  \frac{1}{2}(-|\bm{\xi}|+8)e^{|\bm{\xi}|},
    \end{equation}
      where when $|\xi|>8$, $\frac{1}{\alpha}(-|\bm{\xi}|+8)e^{|\bm{\xi}|} >0$ the training of discriminator is unstable; when  $|\xi|<8$, $\frac{1}{\alpha}(-|\bm{\xi}|+8)e^{|\bm{\xi}|} <0$ the training of generator is unstable. In this case the training for both generator and discriminator is unstable.
    \item For $\alpha = 3$ we have,
    \begin{equation}
      \mathcal{F}((1 + \frac{\|\bm{x}\|^2}{2\alpha})^{-\alpha}) = \mathcal{F}(\frac{1}{(1+\frac{\|\bm{x}\|^2}{6})^3}) =  \frac{1}{4} (|\bm{\xi}|^2 - 3|\bm{\xi}|+3)e^{|\bm{\xi}|},      
    \end{equation}

      where for all $\bm{\xi}$,  $\frac{1}{\alpha} \frac{3}{4} (|\bm{\xi}|^2 - 3|\bm{\xi}|+3)e^{|\bm{\xi}|}>0$. From the above proposition, we know that for $\alpha=3$, the training of generator is stable while the training for discriminator is unstable.
\end{itemize}

To summarize, we have analyzed the stability properties of MMD GANs with a rational quadratic kernel for various values of $\alpha$. Our results indicate that the stability of the GAN training process depends on the specific value of $\alpha$, with some values leading to stable training of the generator while others do not. However, in all cases, the training process of the discriminator is found to be unstable. 
        
\end{proof}

\subsubsection{Stability analysis for EIEG GAN}

\begin{proposition}
In EIEG GAN, the training process of the generator, i.e.,$\min_G \mathcal{L}_G$, is stable, while the training process of the discriminator, i.e., $\max_D \mathcal{L}_D$, is unstable.
\end{proposition}

\begin{proof}
    For EIEG GAN, we have $e(\bm{x},\bm{y}) = \frac{1}{\|\bm{x}-\bm{y}\|^{d-1}}$. The Fourier transform for it is
    \begin{equation}
    \mathcal{F}(\frac{1}{\|\bm{x}\|^{d-1}})  = \frac{1}{|\bm{\xi}|},
    \end{equation}
  where for all $\bm{\xi}$,  $\frac{1}{|\bm{\xi}|}>0$. From the above proposition, we know that for EIEG GAN, the training of generator is stable while the training for discriminator is unstable.
\end{proof}

\subsection{Stability analysis for stabilized particle-based distance GANs}

 We give training stability analysis for our proposed stabilizing method in section \ref{section3.3}.

\subsubsection{Stability analysis for stabilized  MMD GAN with Gaussian RBF kernel}

Rescale  MMD GAN with Gaussian RBF kernel

\begin{equation}
e_{\sigma}(\bm{x},\bm{y}) = \frac{1}{\sigma} \exp \left(-\frac{1}{2\sigma^2}\left\|\bm{x}-\bm{y}\right\|^2\right),
\end{equation}
where $\sigma \in (0,+\infty)$.

In the stabilized function $\mathcal{L}^{s}_D$ (Eq. (\ref{Eqn Loss of sta D}), the stabilized distance is 
\begin{equation}
\widetilde{e}_{\sigma}(\bm{x},\bm{y}) = e_{\sigma_1}(\bm{x},\bm{y}) - \epsilon e_{\sigma_2}(\bm{x},\bm{y}), 
\end{equation}
where $\sigma_2 < \sigma_1$.

\begin{proposition}
For stabilized MMD GAN with Gaussian RBF kernel, the training process of the stabilized discriminator $D$ to $\max \mathcal{L}^{s}_D$ with the stabilized distance $\widetilde{e}_{\sigma}(\bm{x},\bm{y}) = e_{\sigma_1}(\bm{x},\bm{y}) - \epsilon e_{\sigma_2}(\bm{x},\bm{y})$ is stable for $\epsilon > 1$ and $\sigma_1<\sigma_2$. 
\end{proposition}

\begin{proof}
    For the stabilized $\widetilde{e}_{\sigma}(\bm{x},\bm{y}) = e_{\sigma_1}(\bm{x},\bm{y}) - \epsilon e_{\sigma_2}(\bm{x},\bm{y})$, the Fourier transform for it is
    \begin{equation}
    \mathcal{F}(\frac{1}{\sigma_1}e^{-\frac{\|\bm{x}\|^2}{2\sigma_1^2}} - \epsilon \frac{1}{\sigma_2}e^{-\frac{\|\bm{x}\|^2}{2\sigma_2^2}}) = e^{-\frac{\sigma_1^2|\bm{\xi}|^2}{4}} - \epsilon e^{-\frac{\sigma_2^2|\bm{\xi}|^2}{4}} = e^{-\frac{-\sigma_1^2|\bm{\xi}|^2}{4}}(1 - \epsilon e^{-\frac{\sigma_2^2-\sigma_1^2}{4}|\bm{\xi}|^2}),
    \end{equation}

    where since $\sigma_2 < \sigma_1$, $(1 - \epsilon e^{-\frac{\sigma_2^2-\sigma_1^2}{4}|\bm{\xi}|^2}) < 0$ with $\epsilon > 1$. Thus we have $\mathcal{F}(\widetilde{e}_{\sigma}(\bm{x},\bm{y}))<0$ which indicates that the training process of the stabilized discriminator $D$ is stable.

\end{proof}

\subsubsection{Stability Analysis for Stabilized  MMD GAN with rational quadratic kernel}

Rescale MMD GAN with rational quadratic kernel
\begin{equation}
e_{\alpha}(\bm{x},\bm{y}) = \alpha (1 + \frac{\|\bm{x}\|^2}{2\alpha})^{-\alpha},
\end{equation}
where $\alpha \in (0,+\infty)$.

In the stabilized function $\mathcal{L}^{s}_D$ (Eq. (\ref{Eqn Loss of sta D}), the stabilized distance is 
\begin{equation}
\widetilde{e}_{\alpha}(\bm{x},\bm{y}) = e_{\alpha_1}(\bm{x},\bm{y}) - \epsilon e_{\alpha_2}(\bm{x},\bm{y}), 
\end{equation}

where $\alpha_2 > \alpha_1$.

Here we demonstrate the training stability of our stabilized method with a special case that $\widetilde{e}_{\alpha}(\bm{x},\bm{y}) = e_{\frac{1}{2}}(\bm{x},\bm{y}) - \epsilon e_{1}(\bm{x},\bm{y})$.

\begin{proposition}
In  MMD GAN with rational quadratic kernel, the training process of the stabilized discriminator $D$ to $\max \mathcal{L}^{s}_D$ with the stabilized distance $\widetilde{e}_{\alpha}(\bm{x},\bm{y}) = e_{\frac{1}{2}}(\bm{x},\bm{y}) - \epsilon e_{1}(\bm{x},\bm{y})$ is stable for $\epsilon > 1$. 
\end{proposition}

\begin{proof}
    For the stabilized $\widetilde{e}_{\alpha}(\bm{x},\bm{y}) = e_{\frac{1}{2}}(\bm{x},\bm{y}) - \epsilon e_{1}(\bm{x},\bm{y})$, the Fourier transform for it is
    \begin{equation}
    \mathcal{F}(\frac{1}{2}\frac{1}{\sqrt{1+\|\bm{x}\|^2}} - \epsilon \frac{1}{1+\frac{\|\bm{x}\|^2}{2}}) = K_0(\bm{\xi}) - \epsilon e^{|\bm{\xi}|},
    \end{equation}

    where $K_0(\bm{\xi}) > 0$ is Bessel function, and $K_0(\bm{\xi}) \rightarrow +\infty$ as $\xi \rightarrow 0$.

    In this case
    \begin{equation}
       \frac{d \hat{v}}{dt} =   C|\bm{\xi}|^{2}\hat{v}\mathcal{F}(e_{\frac{1}{2}}(\|\bm{x}\|)-\epsilon e_1(\|\bm{x}\|)) = C|\bm{\xi}|^{2}\hat{v}( K_0(\bm{\xi}) - \epsilon e^{|\bm{\xi}|}).
    \end{equation}

Knowing that in the training of $D$, we always normalize the input data $X$ into a finite domain, i.e., $[-1,1] \times [-1,1]$. We know that $|\bm{\xi}|$ is bounded and $\min |\bm{\xi}| = \frac{1}{2}$ for $|\bm{\xi}|$ not equal to 0. When $\epsilon > \frac{K_{0}(\frac{1}{2})}{e^{\frac{1}{2}}}$, $(K_0(\bm{\xi}) - \epsilon e^{|\bm{\xi}|})<0$, the training process is stable. When $|\bm{\xi}| = 0$, $\frac{d \hat{v}}{dt} =C|\bm{\xi}|^{2}\hat{v}( K_0(\bm{\xi}) - \epsilon e^{|\bm{\xi}|}) \rightarrow \infty$. Hence, the training process is unstable when $\bm{\xi} = 0$.

\textbf{Remark.} In this case, a better choice for stabilizing terms may be $e_{m}(\bm{x},\bm{y}) = \frac{1}{\|\bm{x}-\bm{y}\|^m}$. 
\end{proof}

\subsubsection{Stability analysis for stabilized EIEG GAN}

Rescale EIEG GAN

\begin{equation}
e_{m}(\bm{x},\bm{y}) = \frac{1}{\|\bm{x}- \bm{y}\|^m},
\end{equation}
where $m \in (0,+\infty)$.

In the stabilized function $\mathcal{L}^{s}_D$ (Eq. (\ref{Eqn Loss of sta D}), the stabilized distance is 
\begin{equation}
\widetilde{e}_{m}(\bm{x},\bm{y}) = e_{m_1}(\bm{x},\bm{y}) - \epsilon e_{m_2}(\bm{x},\bm{y}), 
\end{equation}

where $m_2 > m_1$.

Here we demonstrate the training stability of our stabilized method with a special case that $\widetilde{e}_{m}(\bm{x},\bm{y}) = e_{d-1}(\bm{x},\bm{y}) - \epsilon e_{d+3}(\bm{x},\bm{y})$, where $d$ is the feature dimension.

\begin{proposition}
For stabilized EIEG GAN, the training process of the stabilized discriminator $D$ to $\max \mathcal{L}^{s}_D$ with the stabilized distance $\widetilde{e}_{\alpha}(\bm{x},\bm{y}) = e_{d-1}(\bm{x},\bm{y}) - \epsilon e_{d+3}(\bm{x},\bm{y})$ is stable for $\epsilon > 1$. 
\end{proposition}

\begin{proof}
    For the stabilized $\widetilde{e}_{\alpha}(\bm{x},\bm{y}) = e_{d-1}(\bm{x},\bm{y}) - \epsilon e_{d+3}(\bm{x},\bm{y})$, the Fourier transform for it is
    \begin{equation}
            \mathcal{F}( \frac{1}{\|\bm{x}- \bm{y}\|^{d-1}} - \epsilon \frac{1}{\|\bm{x}- \bm{y}\|^{d+3}} ) = \frac{1}{|\bm{\xi}|}-\varepsilon |\bm{\xi}|^{3}.
    \end{equation}

    In this case
    \begin{equation}
        \frac{d \hat{v}}{dt} =   C|\bm{\xi}|^{2}\hat{v}\mathcal{F}(e_{d-1}(\|\bm{x}\|)-\epsilon e_{d+1}(\|\bm{x}\|)) = C|\bm{\xi}|^{2}\hat{v}(\frac{1}{|\bm{\xi}|}-\varepsilon |\bm{\xi}|^{3}).
    \end{equation}

Knowing that in the training of $D$, we always normalize the input data $X$ into a finite domain, i.e., $[-1,1] \times [-1,1]$. We know that $|\bm{\xi}|$ is bounded and $\min |\bm{\xi}| = \frac{1}{2}$ for $|\bm{\xi}|$ not equal to 0. When $\epsilon > 1$, $(1 - \epsilon |\bm{\xi}|^4)<0$, the training process is stable. When $|\bm{\xi}| = 0$, $\frac{d \hat{v}}{dt} =0$. Hence, the perturbation does not grow, and the solution remains stable.

\end{proof}

\newpage
\section{Other gradient flow}
While the stability analysis in our main paper focuses on the Wasserstein gradient flow, we recognize that other types of gradient flow can also play an important role of GANs. To extend our stability analysis framework, we present examples where we consider other types of gradient flow as well.

Our proposed stability analysis framework can be applied to analyze the training stability of other types of GANs. The process of stability analysis remains similar: firstly, we identify the gradient flow function corresponding to the subject of the study. Next, we derive the perturbation evolution equation that describes how small perturbations appearing at some time during the training behave. Finally, by analyzing the perturbation evolution equation, we can gain insights into the stability properties of the GAN and the factors that influence them. This approach can provide valuable guidance for improving the training stability of GANs and enhancing their performance in practical applications.

\subsection{Stability analysis for Vanilla GANs}

In this section, we use our framework to analyze Vanilla GAN \cite{refGAN}. In Vanilla GAN, the feature space is $\{0,1\}$. In this case, the probability function $\mathbb{P}_{f_{\mathrm{data}}}\{D(\bm{x})=1, \bm{x}\in \mathcal{A}\} = 1$, $\mathbb{P}_{f_{g}}\{D(\bm{x})=0, \bm{x}\in \mathcal{B}\}=1$, where $\mathcal{A}$ stands for the dataset of data samples and $\mathcal{B}$ stands for the dataset of generated samples. In this case, the probability density function is Delta function at points $1$ and $0$, which is non-smooth. The Wasserstein gradient flow framework,
which describes the evolution of a smooth probability density function, cannot be applied to Vanilla GAN. Thus we analyze the training stability through the particle dynamics in feature space. 

The objective function of Vanilla GAN is
\begin{equation}
    \max_{D} \min_{G}  v(G,D) = \mathbb{E}_{y \sim \mathbb{P}_{\mathrm{data}}} \log[D(\bm{y})] + \mathbb{E}_{\bm{z} \sim \mathcal{N}(\bm{0},I)} \log[1 - D(G(\bm{z}))],
\end{equation}
where $D$ is discriminator and $G$ is generator.

In this case, the corresponding loss function for discriminator is
\begin{equation}
    \max_D \mathcal{L}_{D}=  \mathbb{E}_{\bm{y} \sim \mathbb{P}_{\mathrm{data}}} \log[D(\bm{y})] + \mathbb{E}_{\bm{x} \sim \mathbb{P}_g} \log[1 - D(\bm{x})].
\end{equation}

With G fixed, for a sample $\bm{x_0}$ the evolution of the samples in feature space is
\begin{equation}
    dX_t  = \bigg[\frac{p_r(\bm{x_0})}{X_t} - \frac{p_g(\bm{x_0})}{1-X_t}\bigg]dt.
\end{equation}

For discriminator, if $p_r$ and $p_g$ have disjoint support $\mathcal{A}$ and $\mathcal{B}$, i.e., $\mathcal{A} \cap \mathcal{B} = \emptyset$.
\begin{itemize} 
    \item For the case $\bm{x}_0 \in \mathcal{A}$,  the data samples dynamics is $d X_t  = \frac{p_r(\bm{x}_0)}{X_t} dt$, and the corresponding perturbation evolution equation $$\frac{d v}{dt}  = -\frac{v}{X_t^2}.$$
    \item For the case $\bm{x}_0 \in \mathcal{B}$, the generated samples dynamics is $d X_t  = - \frac{p_g(\bm{x}_0)}{1-X_t}dt $, and the corresponding perturbation evolution equation $$\frac{d v}{dt}  = -\frac{v}{(1-X_t)^2}.$$
    \item For the case $\bm{x}_0 \notin \mathcal{A}$ or $\bm{x}_0 \notin \mathcal{B}$, the dynamics is $dX_t  = 0$, and the corresponding perturbation evolution equation $$\frac{d v}{dt}  = 0.$$
\end{itemize}
Based on the above perturbation evolution equation, the training process of discriminator of Vanilla GAN is stable.

The corresponding loss function for generator is (the $-\log D$ alternative \cite{refPreWGAN})
\begin{equation}
    \min_G \mathcal{L}_{G}= \mathbb{E}_{z \sim \mathcal{N}(\bm{0},I)} \log[1 - D(G(z))].
\end{equation}

With a fixed $D$, for a generated sample $\bm{x_0}$, the evolution of the generated sample in feature space $X_t = D(\bm{x}_0)$ is

\begin{equation}
    dX_t = \bigg[\frac{p_g(\bm{x}_0)}{1-X_t}\bigg]dt.
\end{equation}
 And the corresponding perturbation evolution equation is 
 \begin{equation}
    \frac{dv}{dt} = \frac{v}{(1-X_t)^2},
\end{equation}
which indicates that the training process of generator is unstable.

Thus an alternative loss for the generator is proposed (the $-\log D$ trick \cite{refPreWGAN}.)

\begin{equation}
    \mathcal{L}_{G} =\mathbb{E}_{z \sim \mathcal{N}(\bm{0},I)} [-\log D(G(z))],
\end{equation}

and the evolution of the generated sample $\bm{x_0}$ in feature space is
\begin{equation}
    dX_t = \bigg[ \frac{p_g(\bm{x}_0)}{X_t}\bigg]dt.
\end{equation}

The corresponding perturbation evolution equation is
\begin{equation}
    \frac{dv}{dt} = -\frac{v}{X_t^2},
\end{equation}
which indicates that the training process of the alternative discrimimator is stable.

\subsection{Functional gradient flow}

We can also consider the gradient flow of discriminator $D$ to analyze the training stability. Here we use the gradient flow of $D$ to analyze the training stability of the discriminator in WGAN-GP \cite{refWGANGP}. And this example is used to illustrate adding gradient penalty in the loss function of the discriminator as a stabilizing term.

The loss function of the discriminator $D$ in the WGAN-GP is 
\begin{equation}
\mathcal{L}_D=-\underset{\bm{x} \sim \mathbb{P}_g}{\mathbb{E}}[D(\bm{x})]+\underset{\bm{y} \sim \mathbb{P}_{\mathrm{data}}}{\mathbb{E}}[D(\bm{y})]-\lambda \underset{\hat{\boldsymbol{x}} \sim \mathbb{P}_{\hat{\boldsymbol{x}}}}{\mathbb{E}}\left[\left(\left\|\nabla_{\hat{\boldsymbol{x}}} D(\hat{\boldsymbol{x}})\right\|_2-1\right)^2\right]
\end{equation}

For training of the discriminator $\min \mathcal{L}_D$, we consider the gradient flow of the discriminator
\begin{equation}
        \frac{\partial D}{\partial t} = \frac{\delta \mathcal{L}_{D}}{\delta D} = p_r - p_g+ \lambda(\Delta D - 2 \nabla(\frac{\nabla D}{\|\nabla D \|} )) (\epsilon p_r + (1-\epsilon)p_g),
\end{equation}

The corresponding perturbation evolution equation is
\begin{equation}
\begin{aligned}
    \frac{\partial v}{\partial t} =  \lambda \Delta v (\epsilon p_r(x_0) + (1-\epsilon)p_g(x_0))
\end{aligned}
\label{Eqn: WGAN perturbation evo}
\end{equation}

Taking Fourier transform on both sides of Eqn.(\ref{Eqn: WGAN perturbation evo}), we have
\begin{equation}
    \frac{d \hat{v}}{dt} =
    - \lambda (\epsilon p_r(x_0) + (1-\epsilon)p_g(x_0))|\bm{\xi}|^2 \hat{v},
\end{equation}
which indicates that the evolution of $W^2$ during the training is neutrally stable.

From the above analysis, we know that gradient penalty is also a kind of stabilizing term that can be added in discriminator loss function. And also, with the gradient penalty in the discriminator loss function, the gradient flow of $D$ has a Laplacian term which causes the discriminator to become overly smooth and the generated samples may become connected, leading to mode collapse. The experiment results shown in Fig.\ref{Fig.sub.2g} validate this view.

\subsection{Parameter gradient flow}

To take the structure of the neural network into account in the training stability analysis, we can consider the gradient flow of the corresponding parameter. We can consider a simple discriminator made of a fully connected network without bias term, with input $\bm{x}$:
\begin{equation}
D(\bm{x},\phi) = W^{L+1}a_{L}(W^{L}(a_{L-1}(W^{L-1}(...a_1(W^{1}\bm{x}...))))),
\end{equation}
where $\phi := \{W^{1},...,W^{L},W^{L+1}\}$ is the learning parameters set, $W^{l} \in \bm{R}^{d_l \times d_{l-1}}, W^{L+1} \in \bm{R}^{1 \times d_L}$, and $a_l$ is an piece-wise linear linear activation function. To illustrate the idea, we consider a simple case with a single layer neural network with structure
\begin{equation}
    D(\bm{x},\phi) =  W^{2}a_{1}(W^1\bm{x})
\end{equation}

For training of the discriminator with $\max \mathcal{L}_D$, we consider the gradient flow for parameter $W^2$ is
\begin{equation}
    \frac{\partial W^2}{\partial t} = \frac{d \mathcal{L}_D}{dW^2} = \frac{\delta \mathcal{L}_D}{\delta D} \bm{A}_1(W^{1}\bm{x}),
\end{equation}
where $\bm{A}_1(W^{1}\bm{x}) \in \mathbb{R}^{d1 \times d2}$ is from the derivative of $\frac{\partial D}{\partial W^2}$. Considering perturbations $v\ll1$ appearing during the training, the perturbation evolution equation is
\begin{equation}
   \| \frac{d v}{d t} \|= 0,
\end{equation}
which indicates that the evolution of $W^2$ during the training is neutrally stable.

One of popular stabilizing methods is spectral normalization \cite{refSNGAN}, where $\hat{W}_{SN}(W):= W/\sigma(W)$ and $\sigma(W)$ is the spectral norm of the matrix. 
In this case the gradient flow for the last layer parameter $W^2$ is

\begin{equation}
    \frac{\partial W^2}{\partial t} = \frac{\delta \mathcal{L}_D}{\delta D} \frac{\bm{A}_1(W^{1}\bm{x})}{\sigma(W^2)}.
    \label{Eqn parameterGflow}
\end{equation}
Considering the perturbation evolution equation

\begin{equation}
    \|\frac{dv}{dt}\| = - |\frac{\delta L_D}{\delta D}| \|a_1(W^1\bm{x})\|\frac{1}{\sigma(W^2)^2}\|v\|,
\end{equation}
which indicates that the evolution of $W^2$ during the training is stable.

\subsection{Future Work}

From some of the above discussion, we can know that we can generalize the framework of our stability analysis. If we consider the network structure, we can use the parameter gradient flow (Eq. (\ref{Eqn parameterGflow})) for the stability analysis. In addition to this, we can also take into account the effect of the optimization algorithm used on stability by considering discrete time dynamics.

\newpage
\section{Experimental details} 

In this section, we include more details about the experiments done in the main paper. We show neural network architecture and hyper-parameter settings for them. And we also show that All experiments are conducted on Python 3.7 with NVIDIA 2080 Ti. 

\subsection{Gaussian Mixture}

For Gaussian Mixture, we sample a 2-d 8-cluster Gaussian Mixture distributed in a circle, where cluster means are sampled from $\mathcal{N}(I, \bm{0})$, the marginal probability of each cluster is $1/8$. For MMD GAN and MMD GAN-GP, the linear combination of Gaussian RBF kernel, i.e., $e_{rbf}(\bm{x},\bm{y}) = \sum_{i = 1}^{K}e_{\sigma_i}(\bm{x},\bm{y})$, where $K = 3$ and $\sigma_{i} = \{4,8,16\}$. 
In stabilized MMD GAN, the stabilizing term, we set $s(\bm{x},\bm{y}) = \sum_{i = 1}^{K}e_{\sigma_i}(\bm{x},\bm{y})$, where $\sigma_{i} = \{1,\sqrt{2},2\}$. For both generator and discriminator, we use Adam to train with learning rate $lr = 5e-3$ and $(\alpha,\beta) = (0.5,0.9)$ for $3000$ epochs.

For this case, we use multi-layer perceptron (MLP) networks.
\begin{itemize}
    \item The MLP discriminator takes a 2-dimensional tensor as the input. Its architecture has a set of fully-connected layers (fc marked with input-dimension and output-dimension) and LeakyReLU layers (hyperparameter set as 0.2): \textsl{fc (2 $\rightarrow$ 100), LeakyReLU, fc (100 $\rightarrow$ 50), LeakyReLU, fc (50 $\rightarrow$ 16)}.
    \item  The MLP generator network takes a 2-dimensional random Gaussian variables as the input. Its architecture: \textsl{fc (2 $\rightarrow$ 100), LeakyReLU, fc (100 $\rightarrow$ 50), LeakyReLU, fc (50 $\rightarrow$ 2)}.
\end{itemize}

\subsection{Image generation}

For image generation, we use the dataset CIFAR-10. For this case, we use convolutional neural networks (CNN). For both generator and discriminator, we use Adam to train with learning rate $lr = 5e-5$ and $(\alpha,\beta) = (0.5,0.9)$ for $50$ epochs with batchsize $B = 64$. And we also train $n_c=5$ times discriminator per generator.

\begin{itemize}
    \item The CNN elastic discriminator  takes a $B \times C \times H \times W$ tensor as the input. Its architecture has a set of convolution layers (conv marked with input-c, output-c, kernel-size, stride, padding), Batch Normalization layers (BN) and LeakyReLU layers (hyperparameter as 0.2): \textsl{conv (3,64,4,2,1), LeakyReLU, conv (64,128,4,2,1), BN, LeakyReLU, conv (128,256,4,2,1), BN, LeakyReLU, conv (256,512,4,2,1), BN, LeakyReLU, conv (512,128,4,2,1)}.
    \item  The CNN generator network given a $100$ dimensional random Gaussian variables: \textsl{conv (100,256,4,2,0), BN, ReLU, conv (256,128,4,2,1), BN, ReLU, conv (128,64,4,2,1), BN, ReLU, conv (64,32,4,2,1), Tanh}.
\end{itemize}

\textbf{Quantitative analysis} We also evaluate the FID scores, for MMD GAN: 64.72 and for stabilized MMD GAN: 48.61. The Inception Score, for MMD GAN: 6.14 and for stabilized MMD GAN: 6.8489.

\textbf{Generated samples for CIFAR-10} The generated samples are shown in Fig.\ref{Fig.Image}.

\begin{figure}[!hbtp]
        \centering 
	\begin{subfigure}[t]{0.49\textwidth}
		\centering
		\includegraphics[width= 1\textwidth]{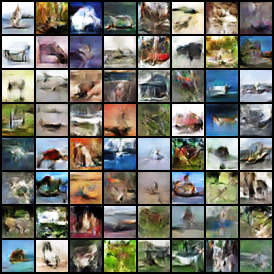}
		\caption{MMD GAN.}
		\label{Fig.AP1.1}
	\end{subfigure}
	\begin{subfigure}[t]{0.49\textwidth}
		\centering
		\includegraphics[width= 1\textwidth]{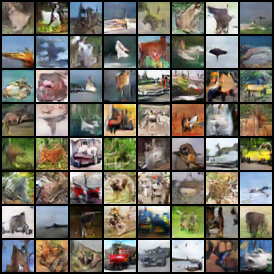}
		\caption{Stabilized MMD GAN.}
		\label{Fig.AP1.2}
	\end{subfigure}
\caption{Uncurated samples on CIFAR-10 datasets. }
\label{Fig.Image}
\end{figure}	

\subsection{Experiments on MMD GAN with rational quadratic kernel}
We also conduct experiments on MMD GAN with rational quadratic kernel with CIFAR-10 to show its instability training process.

\begin{figure}[!hbtp]
        \centering 
	\begin{subfigure}[t]{0.85\textwidth}
		\centering
		\includegraphics[width= 1\textwidth]{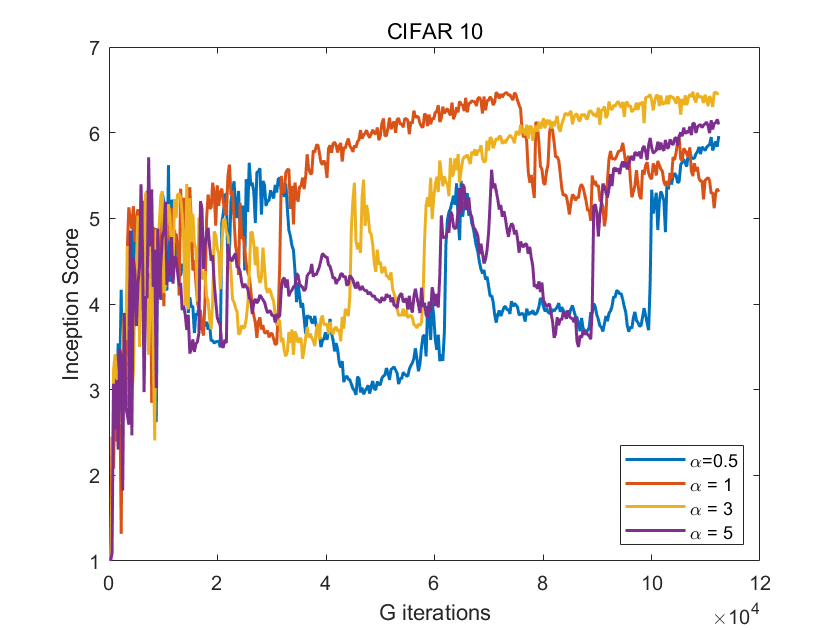}
		\label{Fig.AP11}
	\end{subfigure}
\caption{Training Curves on CIFAR10 of MMD GAN with rational quadratic kernel.}
\label{Fig. Image}
\end{figure}

\noindent

\end{appendices}


\end{document}